\documentclass[12pt]{article}
\usepackage[top=1in,bottom=1in, left=1in, right=1in]{geometry}
\usepackage[utf8]{inputenc}
\usepackage{tabularx}
\usepackage{bbm}
\usepackage{mathtools}
\usepackage{enumitem}
\usepackage[english]{babel}
\usepackage{amssymb}
\usepackage{amsmath}
\usepackage{xcolor}
\usepackage{url}
\usepackage{relsize}
\usepackage{breqn}
\usepackage{comment}
\usepackage{amsmath,amsthm,amssymb,amsopn,amsfonts,pdfpages,dsfont}
\usepackage{graphics}
\usepackage{graphicx}
\usepackage{bbm,subfig}
\usepackage{caption,url}
\usepackage[colorlinks=true,linkcolor=cyan,citecolor=blue]{hyperref}
\usepackage{booktabs}
\usepackage{stackengine}
\usepackage{authblk}
\usepackage{verbatim}
\usepackage{bm}
\usepackage{bbold}
\usepackage{setspace}
\usepackage{blkarray}
\usepackage{multirow}
\usepackage{mwe}
\usepackage{subfig}
\usepackage[ruled,vlined]{algorithm2e}

\newtheorem{definition}{Definition}

\newtheorem{rem}{Remark}
\newtheorem{theorem}{Theorem}

\newcommand{\bX}{\mathbf{X}}

\newcommand{\bB}{\mathbf{B}}
\newcommand{\bU}{\mathbf{U}}
\newcommand{\bA}{\mathbf{A}}

\newcommand{\bZ}{\mathbf{Z}}
\newcommand{\bY}{\mathbf{Y}}

\newcommand{\hbX}{\hat{\mathbf{X}}}
\newcommand{\hbb}{\hat{\mathbf{B}}}
\newcommand{\hbbc}{\hat{\mathbf{B}}^{(c)}}
\newcommand{\hbbct}{\hat{\mathbf{B}}^{(c,t)}}
\usepackage{todonotes}

\title{\textbf{Adversarial contamination of networks in the setting of vertex nomination: a new trimming method }}
\author{Sheyda Peyman$^1$, 
Minh Tang$^2$,
Vince Lyzinski$^1$, \\
\vspace{3mm}
\small{$^1$University of Maryland, College Park, Department of Mathematics\\
  $^2$North Carolina State University, Department of Statistics\\}
}
\date{}

\begin{document}

\maketitle
\begin{abstract}
As graph data becomes more ubiquitous, the need for robust inferential graph algorithms to operate in these complex data domains is crucial. In many cases of interest, inference is further complicated by the presence of adversarial data contamination. The effect of the adversary is frequently to change the data distribution in ways that negatively affect statistical and algorithmic performance. We study this phenomenon in the context of vertex nomination, a semi-supervised information retrieval task for network data. Here, a common suite of methods relies on spectral graph embeddings, which have been shown to provide both good algorithmic performance and flexible settings in which regularization techniques can be implemented to help mitigate the effect of an adversary.  Many current regularization methods rely on direct network trimming to effectively excise the adversarial contamination, although this direct trimming often gives rise to complicated dependency structures in the resulting graph. We propose a new trimming method that operates in model space which can address both block structure contamination and white noise contamination (contamination whose distribution is unknown). This model trimming is more amenable to theoretical analysis while also demonstrating superior performance in a number of simulations, compared to direct trimming. 
\end{abstract}

\section{Introduction}

Graph-valued data arises in numerous diverse scientific fields ranging from sociology, epidemiology and genomics to neuroscience and economics.
For example, sociologists have used graphs to examine the roles of user attributes (gender, class, year) at American colleges and universities through the study of Facebook friendship networks \cite{facebook} and have studied segregation and homophily in social networks \cite{mele};
epidemiologists have recently modeled Human-nCov protein-protein interactions via graphs \cite{covid}, and neuroscientists have used graphs to model neuronal connectomes \cite{eichler2017complete}. 
As graphs have become more prevalent, there has been a need for robust graph algorithms to operate in these complex data domains.

In many cases of interest, inference is further complicated by the presence of adversarial data contamination designed to reduce algorithmic effectiveness.
Some examples include attacks in Graph Neural Networks (GNN) via gradient-based attack procedures \cite{adversarial_large_graphs} and via reinforcement learning \cite{adverDL2} to name a few (for a review of adversarial attacks and defenses on graphs in Deep Neural networks (DNNs) see \cite{adversarial_survey}).
Often, the effect of the adversary is to change the data distribution in ways that negatively affect statistical inference and algorithmic performance. 
In this paper we study this phenomenon in the context of the \textit{vertex nomination} (VN) inference task \cite{CopPri2012,Coppersmith2014,marchette2011vertex,fishkind2015,lyzinski2017consistent}, a semi-supervised information retrieval task akin to personalized recommender systems \cite{resnick1997recommender} on graphs. 
Succinctly, the vertex nomination problem can be stated as ``given a pair of graphs $G_1$ and $G_2$ and a set of vertices of interest in $G_1$, find a corresponding set of vertices of interest in $G_2$ and create a rank list of the vertices in $G_2$" \cite{patsolic2017vertex} (the original, single graph analogue tasked the user with finding additional vertices of interest in $G_1$ given a training set of vertices of interest in $G_1$).
The corresponding vertices of interest in $G_2$ should ideally (if they exist) appear at the top of the nomination list.
In the recent literature, algorithms for approximately solving the vertex nomination task have been implemented in myriad applied problem spaces such as
finding fraud in the Enron email corpus \cite{CopPri2012,marchette2011vertex}, 
identifying web advertisements associated with human trafficking \cite{fishkind2015},
identifying search queries across Bing transition rate networks \cite{agterberg2019vertex, zheng_vn}, and
identifying neurons across hemispheres in a Drosophila larva brain connectome \cite{Bing}, among others.

While the theoretical and practical impacts of adversarial noise and subsequent regularization have been widely studied in graph-valued data applications (see, for example, \cite{advCD,cai2015robust,advDL,adverDL2,entezari2020all} among myriad others), 
the development of these concepts in vertex nomination is relatively nascent and is an area of current research. 
In \cite{agterberg2019vertex}, the effect of adversarial data contamination was introduced and studied in the context of vertex nomination.
In addition to developing a theoretical basis for understanding the action of an adversary in vertex nomination,  the authors in \cite{agterberg2019vertex} empirically demonstrate the expected cycle of performance degradation due to adversarial noise followed by data regularization (via the trimming method inspired by \cite{edge2018trimming}) recovering much of the lost performance.
The adversarial contamination model in \cite{agterberg2019vertex} is formulated as a probabilistic mechanism acting on the network via edge/vertex deletion or addition.
The goal of the adversary is to move the distribution out of the consistency class of a given vertex nomination rule (see \cite{lyzinski2017consistent}), and thus diminish algorithmic performance.
Within the context of stochastic blockmodel graphs \cite{holland1983stochastic}, a noise model (initially proposed in \cite{cai2015robust}) is introduced in \cite{agterberg2019vertex} in which the addition/deletion of edges and vertices in the network effectively introduces ``noise" blocks into the original blockmodel distribution.
A network regularization method is then introduced which seeks to trim the noise blocks via a network analogue of the classical trimmed-mean estimator \cite{stigler1973asymptotic,huber2004robust} for outlier contaminated data (see Section \ref{sec:blkreg}  for detail).
While the trimming regularization of \cite{agterberg2019vertex} is 
empirically demonstrated to be effective in both real and synthetic applications, it is both difficult to analyze theoretically and practically limited, as it is designed to combat a specific noise-type.

Building upon this work, we consider the effect on vertex nomination performance of a combination of both structured block adversarial noise (as defined in \cite{agterberg2019vertex,cai2015robust}) and diffuse white noise.
Situating our exploration in the context of random dot product graphs (RDPGs) \cite{ hoff2002latent,young2007random,athreya2017statistical} and stochastic blockmodel graphs, 
we propose multiple (theoretically more tractable) regularization methods that can be combined to combat very general noise settings (see Section \ref{general_setting}), and show empirically superior performance. 

\subsection{Vertex Nomination}
\label{sec:VN}
Informally, the vertex nomination (VN) problem we pursue herein can be stated as follows:  Given vertices of interest $S_1$ in a graph $G_1$ and a second graph $G_2$ with a portion of its vertices being also of interest, $S_2\subset V(G_2)$, produce a rank list of the vertices of $G_2$ with the unknown vertices in $S_2$ ideally concentrating at the top of the rank list.
What defines vertices as ``interesting" is vague above, and this is intentional.
Indeed, we allow the user broad leeway in defining what makes a vertex interesting, whether it be membership in a community of interest \cite{fishkind2015}, vertices involved in illicit activity, or vertices corresponding to a particular user in a social network \cite{patsolic2017vertex}.
While a formal definition of the above is presented in \cite{lyzinski2017consistent,agterberg2019vertex}, we do not present the full formal definition here as it would introduce a needless notational complexity, and the informal definition suffices for our present purposes.

Early work in VN considered community membership as the trait defining interesting vertices as interesting  \cite{CopPri2012,Coppersmith2014,marchette2011vertex}, with notable applications including nominating fraudulent activity at Enron 
\cite{CopPri2012}, recommending items in an entity transition graph using data from Bing \cite{Bing} and helping identify websites involved in human trafficking \cite{fishkind2015}.
Rich theory establishing the notions of consistency and Bayes optimality were derived in this community-focused setting \cite{fishkind2015,lyzinski2016consistency,yoder2018vertex}.
More recent work defines the problem of nominating across networks, with interestingness defined more broadly---for example a particular user or collection of users across social networks  \cite{patsolic2017vertex} 
or vertices corresponding to a particular neuron/region in a connectome \cite{levin2020role} might be the vertices of interest. 
Theory was developed in \cite{lyzinski2017consistent, agterberg2019vertex}, where the notion of consistent vertex nomination schemes is defined and developed, with the role of features being further explored in \cite{levin2020role}. 
In addition to problem formulation and theory,
a significant number of methods have been developed to practically tackle the vertex nomination problem including those based on spectral decomposition \cite{yoder2018vertex,fishkind2015,zheng_vn}, likelihood maximization \cite{fishkind2015}, localized graph matching \cite{patsolic2017vertex} Bayesian MCMC \cite{bays}, and specialized ILP formulations \cite{Bing}, to name a few. 

In \cite{agterberg2019vertex}, our motivating work for adversarial VN, the authors considered the VN problem situated across a pair of networks with latent community structure.
Their contamination model (see Section \ref{sec:blocknoise}) corrupted the community communication probabilities and memberships by introducing into each community anomalous vertices with anomalous connection probabilities.
In light of this, a natural model to situate our initial VN analysis is the stochastic blockmodel of \cite{holland1983stochastic} (see Definition \ref{def:sbm}) which posits a simple network model with latent community structure.
This model allows us to handle both the network with community structure and community-structured noise (as considered in \cite{agterberg2019vertex}).
In some settings, more nuanced noise structures (that depend on features beyond community structure) may be more appropriate, and so we will also consider in our analysis the generalized random dot product graph of \cite{young2007random,rubin2017statistical} in our contamination modeling.
The generalized random dot product graph (see Definition \ref{def:RDPG}) will allow us to consider more ``diffuse" contamination frameworks (see Section \ref{sec:difcon}), including white-noise settings and manifold-structured noise settings as well. 
Before delving deeper into the VN problem setting, these two contamination models, and our novel methodologies for regularizing this noise, we will first formally introduce the stochastic blockmodel and generalized random dot product graph models in the next section.

\section{Background}
\label{sec:BG}

We now present 
the setting we work in and the relevant mathematical definitions used in our methods. 
We begin by giving a brief introduction to the stochastic blockmodel and generalized random dot product graph model, followed by presenting the embedding method we choose to work with, namely spectral embedding of the adjacency matrix (ASE). 

\subsection{The Stochastic Blockmodel and Random Dot Product Graphs}
\label{sec:SBM_RDPG}

Random graph models allow us to situate our analysis in the context of traditional statistical inference. 
A host of random graph models have been proposed in the literature (see \cite{kolaczyk2014statistical,kolaczyk2009statistical}), and two of the more popular models in the statistical network inference community are the stochastic blockmodel and the random dot product graph model.
The stochastic blockmodel (SBM), introduced in \cite{holland1983stochastic}, provides a simple model for networks with latent community structure, 
and the SBM and its variants (degree corrected SBM \cite{karrer2011stochastic}, mixed membership SBM \cite{airoldi2008mixed}, hierarchical SBM \cite{lyzinski2016community,peixoto2014hierarchical}, etc.) have been popular models for exploring inference tasks such as community detection/clustering \cite{rohe2011spectral,suss,zhao2012consistency,amini2013pseudo,newman2016equivalence} and community testing \cite{bickel2016hypothesis,lei2016goodness}.
Furthermore, while rather simple, SBMs can also be viewed as an analogue of network histograms and thus provide a universal representation for unlabeled graphs \cite{olhede2014network}.

\begin{definition}[Stochastic Block Model (SBM) with sparsity parameter $\nu$]
\label{def:sbm}
A random graph $G = (V,E)$ on $n$ vertices is distributed according to a Stochastic Blockmodel with parameters K, $\bB$, $\pi$, and sparsity parameter $\nu$ (written $G\sim$SBM$(n,K,\bB,\pi,\nu)$) if the following hold:
\begin{itemize}
    \item[i.] Each vertex $i \in \{1,...,n\}$ is independently assigned to a community/block $\{1,2,...,K\}$ according to the probability vector $\pi\in\mathbb{R}^K$; we will denote the block membership of vertex $v\in V$ via $b_v$. 
    \item[ii.] $\bB=[B_{ij}] \in [0,1]^{K\times K}$, the block probability matrix, is a $K \times K$ symmetric matrix, whose entries provide (up to the sparsity factor $\nu$) the probability of a vertex in one block communicating with a vertex in another block;  $\nu\in[0,1]$ is a sparsity factor controlling the graph density.
    \item[iii.] Conditional on the block-membership for each vertex, the (undirected) edges of the graph are independently sampled according to:
   $$\text{If $u,v\in V$, then
    $\mathbb{1}_{u\sim_{G}v}\sim\text{Bernoulli}(\nu B_{b_v,b_u}).$}$$
\end{itemize}
\end{definition}
\noindent When working across pairs of SBMs, it is often convenient to work within the context of an SBM model that allows for structured dependence across the edges of multiple networks; towards this end, we next introduce the $\rho$-correlated Stochastic Blockmodel from \cite{fishkind2019seeded}.
\begin{definition}[Sparse $\rho$-Correlated Stochastic Block Model (SBM)]
\label{def:rhoSBM}
A pair of graphs $(G_1,G_2)$, is an instantiation of a correlated Stochastic Blockmodel with parameters K, $\bB$, $\pi$ and $\rho$ and with sparsity parameter $\nu$ (written $(G_1,G_2)\sim\text{SBM}(n,K,\bB,\pi,\nu,\rho)$) if the following hold:
\begin{itemize}
    \item[i.] Marginally, $G_1 \sim SBM(n,K,\bB,\pi,\nu)$ and $G_2 \sim SBM(n,K,\bB,\pi,\nu)$.
    \item[ii.] Conditional on the block-membership for each vertex in each graph, the collection of the following indicator random variables is mutually independent,
$$ \left\{ \{ \mathds{1}_{u\sim_{G_1}v}\}_{\{u,v\}\in\binom{V}{2}} \cup \{\mathds{1}_{u \sim_{G_2}v} \}_{\{u,v\}\in\binom{V}{2}} \right\}$$
 except that for each $\{ u,v \} \in \binom{V}{2}$, the correlation between
$(\mathbb{1}_{u \sim_{G_1}v})$ and $(\mathbb{1}_{u \sim_{G_2}v})$ is $\rho.$
\end{itemize}
\end{definition}
\begin{rem}
There is an alternate parameterization of the SBM that we will find convenient in experiments, namely the case when the block sizes are fixed.  
In this case (written $G\sim \mathrm{SBM}(n,K,\bB,\vec n,\nu)$), the block probability assignment vector $\pi$ is replaced by $\vec n\in\mathbb{Z}^K$ satisfying $n_i\geq 0$ for all $i\in[K]$ and $\sum_{i=1}^Kn_i=n$.
Here, vertices are preassigned into the $K$ blocks (so that $|\{v:b_v=i\}|=n_i$) and edges are conditionally independent given these assignments.
The remainder of the definition is essentially unchanged.
\end{rem}

Another popular network model in the statistical network inference literature is the Generalized Random Dot Product Graph (GRDPG), which posits that the edge connectivity is a function of latent vertex attributes that are (potentially) more general than simple community membership \cite{young2007random,athreya2017statistical,rubin2017statistical}.
In the GRDPG setting, inference often begins with estimation of the latent positions \cite{athreya2017statistical}, as these estimates often provide low-dimensional Euclidean representations for the graph at the vertex level.
Given sufficient control over the estimation error of the latent positions \cite{cape2019two}, various inference tasks can be profitably pursued in the embedding space, including clustering \cite{rohe2011spectral,suss,lyz}, classification \cite{sussman12:_univer,tang2013universally}, and testing \cite{tang2017semiparametric,tang2017nonparametric,alyakin2020correcting,agterberg2020nonparametric}, among others. 
\begin{definition}[Generalized Random Dot Product Graph (GRDPG) with sparsity parameter $\nu$]
\label{def:RDPG}
  Let $d \geq 1$ be given and let $\mathcal{X}$ be  a  subset  of $\mathbb{R}^{d}$ such that
  $x^{\top} I_{p,q} y \in [0,1]$. Here $I_{p,q}$ is
  a $d \times d$ diagonal matrix with diagonal entries containing $p$
  ``+1's'' and $q$ ``-1's'' for integers $p \geq 1, q \geq 0$, $p + q = d$. 
 Let $F$ be a distribution supported on $\mathcal{X}\subset \mathbb{R}^d$.
 A random $n$-vertex graph $G = (V,E)$ is distributed according to a Generalized Random Dot Product Graph with parameters $\bX$ and $F$ and sparsity parameter $\nu$ (written $G\sim \mathrm{GRDPG}(\bX,F,\nu)$) if the following hold: 
 
\begin{itemize}
    \item[i.]  $\bX$ is a $ n\times d$ matrix whose rows $X_1,\cdots,X_n$ are i.i.d.\@ random vectors sampled from $\mathcal{X}$ according to $F$.
    \item[ii.] Conditional on $\bX$, the (undirected) edges of the graph are independently Bernoulli random variables with $\mathbb{P}(u\sim_{G}v) = \nu X_u^{\top} I_{p,q} X_v$.
    Written compactly (where $\bA$ is the adjacency matrix of $G$),
    \begin{equation}
        \label{eq:rdpg}
        \mathbb{P}(A|\bX)=\prod_{\{i,j\}\in\binom{V}{2}}(\nu X_i^{\top} I_{p,q} X_j)^{A_{i,j}}(1-\nu X_i^{\top} I_{p,q} X_j)^{1-A_{i,j}}.
    \end{equation}
\end{itemize}
\end{definition}

\noindent Similar to the correlated SBM setting, the following model from \cite{patsolic2017vertex,zheng_vn} allows us to work with pairs of correlated GRDPGs.
\begin{definition}[Sparse $\rho$-Correlated GRDPG]
\label{def:rhoRDPG}
A pair of graphs $(G_1,G_2)$, is an instantiation of a correlated GRDPG with parameters $F$, $\bX$, and $\rho$ and with sparsity parameter $\nu$ (written $(G_1,G_2)\sim \mathrm{GRDPG}(\bX,F,\nu,\rho)$) if the following hold:
\begin{itemize}
\item[i.]  $\bX$ is a $ n\times d$ matrix whose rows $X_1,\cdots,X_n$ are i.i.d.\@ random vectors sampled from $\mathcal{X}$ according to $F$.
    \item[ii.] Marginally, $G_1 \sim \mathrm{GRDPG}(\bX,F,\nu)$ and $G_2 \sim \mathrm{GRDPG}(\bX,F,\nu)$.
    \item[iii.] Conditional on $\bX$, the collection of the following indicator random variables is mutually independent,
$$ \left\{ \{ \mathds{1}_{u\sim_{G_1}v}\}_{\{u,v\}\in\binom{V}{2}} \cup \{\mathds{1}_{u \sim_{G_2}v} \}_{\{u,v\}\in\binom{V}{2}} \right\}$$
 except that for each $\{ u,v \} \in \binom{V}{2}$, the correlation between
$(\mathbb{1}_{u \sim_{G_1}v})$ and $(\mathbb{1}_{u \sim_{G_2}v})$ is $\rho.$
\end{itemize}
\end{definition}
\noindent We note here that the GRDPG model encompasses the SBM model and its popular variants (degree-corrected, hierarchical, etc.) as well as any (conditionally) edge independent random graph for which the edge probabilities matrix is low rank. 
Furthermore, any latent position graph \cite{hoff2002latent} on $n$ vertices can be approximated by a GRDPG with latent positions $\bX \in \mathbb{R}^{n \times d}$ for some sufficiently large $d$.

\begin{rem}
\label{indef_remark} 
\emph{In the context of the GRDPG latent position random graph models, two different sources of nonidentifiability arise naturally: \textit{subspace nonidentifiability} and \textit{model-based nonidentifiability} \cite{agterberg2020two}.
Recall that
the edge probability matrix for a GRDPG is given by $\mathbf{P}=\nu \bX I_{p,q}\bX^T$ for some $n \times d$ matrix $\bX$.
\textit{Subspace nonidentifiability} arises in the context of the non-uniqueness of the eigenbasis of the subspaces corresponding to repeated eigenvalues; i.e., we cannot hope to exactly recover the columns of $\bX$ (i.e., the scaled eigenvectors of $\mathbf{P}$) corresponding to repeated eigenvalues of $\mathbf{P}$. 
More pressing here is the issue of \textit{model nonidentifiability}; specifically, transformations to the inputs $\bX$ under which $\mathbf{P}$ is invariant. More specifically for any indefinite orthogonal matrix $\mathbf{W}_{p,q}$ (so that $\mathbf{W}_{p,q} I_{p,q} \mathbf{W}_{p,q}^{T}=I_{p,q}$), we have $\mathbf{P}=\bX I_{p,q}\bX ^{T}=\bX  \mathbf{W}_{p,q}I_{p,q} \mathbf{W}_{p,q}^T \bX ^{T}$}.
\end{rem}

\subsection{Adjacency Spectral Embedding}
\label{sec:ASE}

In the setting of GRDPGs and more general latent position random graphs, spectral embedding-based methods have proven effective at estimating the latent vertex features $X_i$.
Two popular spectral embedding techniques are the Laplacian Spectral Embedding (LSE) \cite{rohe2011spectral,tang2018limit} and the Adjacency Spectral Embedding (ASE) \cite{suss}.
Herein, we will focus our attention on the ASE.
\begin{definition}[Adjacency Spectral Embedding (ASE)]
\label{def:ASE}
Given the adjacency matrix $\bA$ of an undirected graph, the d-dimensional adjacency spectral embedding of A is defined as follows:
\begin{equation}
\text{ASE}(\bA,d):=\hat{\bX}_d = \mathbf{U} \Sigma^{1/2}\in\mathbb{R}^{n\times d},
\label{ase_def}
\end{equation}
where in the above expression, 
\begin{itemize}
    \item[i.] The singular value decomposition (SVD) of $|\bA|$ is given via
$$|\bA|=(\bA^T\bA)^{1/2} = [\bU|\bU^\perp] [\Sigma\oplus \tilde\Sigma] [\bU|\bU^\perp]^T;$$
\item[ii.] $\Sigma\in\mathbb{R}^{d\times d}$ is the diagonal matrix whose diagonal entries correspond to the $d$ largest singular values of $|\bA|$. 
\item[iii.] $\bU\in\mathbb{R}^{n\times d}$ is the $n \times d$ matrix whose orthonormal columns correspond to the singular vectors of $|\bA|$ associated with the singular values in $\Sigma$;
\end{itemize}
\end{definition}
\noindent ASE provides a theoretically tractable embedding of $G$, with consistency \cite{suss} and central limit theorems \cite{athreya2017statistical,rubin2017statistical} both having been proven for the ASE estimating the underlying latent positions $\bX$ in the GRDPG setting.
Note that one key model parameter that must be estimated when practically computing the ASE is the embedding dimension $d$.
Herein, we will estimate $d$ by choosing an elbow in the scree plot of the singular values of $\bA$ \cite{zhu2006automatic}, a (principled) heuristic outlined in \cite{athreya2017statistical}, and then
take $p$ and $q$ to be the number of positive and negative eigenvalues of $\bA$ corresponding to these leading singular values.

\begin{rem}
\emph{As the ASE can only recover latent positions up to an indefinite orthogonal transform, methods that are invariant to such transforms (e.g., spectral clustering) are unaffected by the nonidentifiability here.
Indeed, interpoint distances are \textit{essentially} preserved; see \cite{agterberg2020two} for detail.}
\end{rem}

\section{Contamination and regularization in VN}
\label{sec:cont}
In order to study the effect of contamination (and regularization) on VN performance, we will adopt the following graph model for $(G_1,G_2)$ moving forward.
We posit that $(G_1,G_2)\sim \mathrm{SBM}(n, K, \bB, \pi,\rho, \nu)$ is the uncontaminated base graph pair. 
This functions to endow a natural notion of vertex correspondence between the vertices of $G_1$ and $G_2$ (as induced by the correlation $\rho$) and hence a canonical definition of vertices of interest in $G_2$ for any subset of vertices $S_1$ in $G_1$.
We further assume that we are given $G_1$, $S_1$ and $G_2^c$, where $G_2^c$ is the network $G_2$ contaminated by one of the stochastic noise models outlined below.

\subsection{Block Contamination}
\label{sec:blocknoise}

Our first noise model, inspired by \cite{cai2015robust} and presented as in \cite{agterberg2019vertex}, corrupts the block structure of our underlying blockmodel network $G_2$.
In brief, this model operates as follows.  Given parameters $\pi_+, \pi_-, s_+, s_-$ (which can vary with $n$, to account for sparsity):
\begin{itemize}
\item[i.] Initialize $E(G_2^c)=E(G_2)$.
    \item[ii.] Create the random set $W_+$ by independently selecting each vertex to be in $W_+$ with some probability $\pi_+$.
    Given $W_+$, independently create the random set $W_-$ by independently selecting each vertex in $V\setminus W_+$ to be in $W_-$ with some probability $\pi_-$.
    \item[iii.] For each vertex pair $\{ v,u\} \in W_+ \times (V \setminus W_-)$:
          \begin{enumerate}
              \item If $\{v,u \} \in E(G_2^c)$, nothing happens.
              \item If $\{ v,u\} \notin E(G_2^c)$, an edge is independently \textit{added} connecting $\{ v,u\}$ in $G_2^c$ with probability $s_+$.
          \end{enumerate}
          \item[iv.] For each vertex pair $\{ v,u \} \in W_- \times (V \backslash W_+)$,
          \begin{enumerate}
              \item If $\{ v,u\} \notin E(G_2^c)$, nothing happens.
              \item If $\{ v,u\} \in E(G_2^c)$, the edge is independently \textit{deleted} from $G_2^c$ with probability $s_-$.
          \end{enumerate}
\end{itemize}
Note that this noise model acts on $G_2$ by adding and removing edges amongst subsets of the vertices.
In this sense, this is an edge contamination model.  However, by considering the portion of the graph that is uncorrupted (edges amongst the vertices not in $W_+ \cup W_{-}$) as a core network correlated to the corresponding core part of $G_1$, we can envision this noise model as adding vertices and edges to the core in order to corrupt the signal.

\begin{rem}
\emph{
In a dense SBM setting (where $\nu=1$) with 2 blocks, consider an $n$-vertex stochastic blockmodel with block probability matrix given by
  \begin{align*}
    \bB =  \begin{pmatrix}
    \mathbf{p} & \mathbf{r} \\
    \mathbf{r} & \mathbf{q}
    \end{pmatrix} 
    \end{align*}
 where, wlog, $p \geq q \geq r \geq 0$.
 In this setting, the above adversarial model gives rise to a new $6$-block stochastic blockmodel, whose block probability matrix $\bB^c$ is given by
\[ 
  \bB^c = 
  \begin{blockarray}{ccccccc}
  \tilde{B}_1 & 
  \tilde{B}_1^{+} &  
  \tilde{B}_1^{-} & 
  \tilde{B}_2  & 
  \tilde{B}_2^{+} & 
  \tilde{B}_2^{-} \\
  \begin{block}{(cccccc)c}
    \textbf{p} & x_{1} & x_{2} & \textbf{r} & x_{3} & x_{4} & \tilde{B}_1  \\
     x_{1} & x_{1} & p & x_{3} & x_{3} & r &\tilde{B}_1^{+} \\
     x_{2} & p & x_{2} & x_{4} & r & x_{4} & \tilde{B}_1^{-} \\
     \textbf{r} & x_{3} & x_{4} & \textbf{q} & x_{5} & x_{6} & \tilde{B}_2 \\
     x_{3} & x_{3} & r & x_{5} & x_{5} & q & \tilde{B}_2^{+} \\
     x_{4} & r & x_{4} & x_{6} & q & x_{6} & \tilde{B}_2^{-}\\
   \end{block}
  \end{blockarray}
\]
 where 
 \begin{align*}
 & x_1 := p + s_{+}(1-p)\\
 & x_2 := p(1-s_{-})\\
 & x_3 := r + s_{+}(1-r)\\
 & x_4 := (1-s_{-})r\\
 & x_5 := q + s_{+}(1-q)\\
 & x_6 := q(1-s_{-})
\end{align*}
Letting $B_1$ and $B_2$ denote the blocks in the original SBM, 
in the above, $\tilde B_1^{+}$ are the vertices in $W_+ \cap B_1$; $\tilde B_1^{-}$ are the vertices in $B_1 \cap W_-$; and $\tilde{B_1}$ 
are the vertices in $B_1 \setminus ( \tilde B_1^{-}\cup \tilde B_1^{+} )$. $\tilde{B_2}$ is defined analogously.
Note first that the induced subgraph amongst $\tilde B_1\cup  \tilde B_2$ is an SBM with block probability matrix $\bB$, and we will often consider this ``core'' of the contaminated $G_2$ to be correlated to the correspondingly structured graph of $G_1$. 
Also note that, given a $K$-block SBM, this contamination model yields a $3K$-block contaminated SBM.
}
\end{rem}

\subsection{Block Regularization}
\label{sec:blkreg}

\begin{figure}[t!]
    \centering 
    \includegraphics[width=0.8\textwidth]{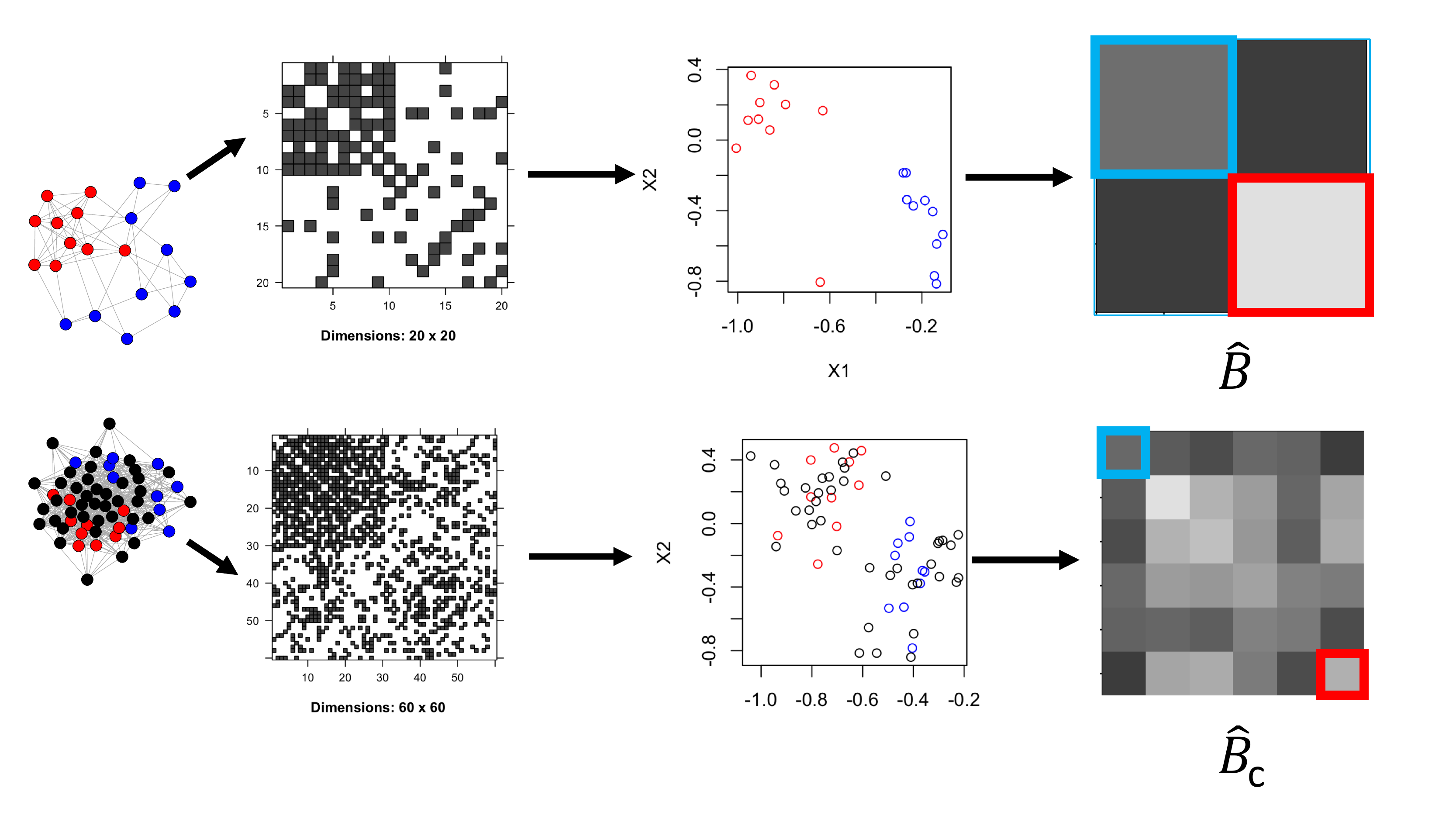}
    \caption{Block contamination and regularization pipeline in a simple 2-block SBM model.  
    The top row represents the clean $G_1$, the bottom row the contaminated $G_2$.  The aligned blocks in the matching step ideally recover the true correspondence across the red communities and across the blue communities (i.e., across the uncontaminated communities).
    }
    \label{fig:blockcontam}
\end{figure}

\noindent To mitigate the effect of the adversary in the above-described setting, \cite{agterberg2019vertex}
proposes a graph-trimming-based regularization method that is  empirically demonstrated to retrieve much of the original inferential performance. 
Once trimmed, the procedure they use to nominate first separately computes the ASE of the two graph, 
then uses seeded vertices to align the networks in the embedding space via orthogonal Procrustes analysis \cite{procr}. 
The vertices are then jointly clustered using model-based Gaussian mixture modeling (here BIC-penalized GMM as in \cite{mclust}). 
Finally, candidate matches/vertices of interest are ranked based on increasing Mahalanobis distance to the vertex (or vertices) of interest \cite{agterberg2019vertex}.
One of the crucial steps in the above-described procedure is precisely the regularization method employed, which is the network analogue of the classical trimmed mean estimator from robust statistics. 
This approach seeks to trim the top $h\%$ and bottom $\ell\%$ of vertices ordered by degree, where $(h,\ell)$ is chosen adaptively via a modularity maximization procedure. 

The trimming in the above method happens \textit{before} creating the adjacency spectral embeddings of the graphs, and the impact of the trimming on the distribution of the ASE is difficult to theoretically parse due to the complicated dependency structures that appear as a side product of the regularization.
While \cite{le2017concentration} have shown that regularization enforces concentration in sparse random graphs in the spectral norm, similar concentration results for the type of row-wise norms (i.e., the $2\to\infty$ norm) that would be needed for a finer-grained analysis are still open. 
In light of the myriad works proving consistency of the ASE for estimating the latent position parameters in random dot product graphs \cite{athreya2017statistical,rubin2017statistical},
trimming after embedding the graphs seems a plausible alternative method to avoid inference complications.

Our new method is based precisely on this point.
We first embed the graphs and estimate block/community structure (i.e., estimating the $\bB$ matrix in an SBM setting) for each graph in spectral space.
The networks are then aligned in model space (via graph matching the estimated $\bB$ matrices), yielding a subgraph of $G_2$ corresponding to $G_1$, with the remainder of $G_2$ being trimmed.
This induced subgraph is re-embedded 
and aligned to the 
ASE of $G_1$. 
The same ranking procedure as in the first method is used to rank the candidate matches in $G_2$. Below we give a more detailed description of our regularization method (see Algorithm \ref{alg:block_reg} for pseudocode).
\begin{itemize}
   \item[1.] Separately embed $G_1$ (into $\mathbb{R}^{d_1}$) and $G_2$ (into $\mathbb{R}^{d_2}$) via ASE, where $d_1$ and $d_2$ are estimated as in Section \ref{sec:ASE}. Estimate $p_1, q_1$ by counting the number of positive (respectively negative) eigenvalues of the adjacency matrix of $G_1$. Estimate $p_2,q_2$ by following a similar procedure.
   \item[2.] Separately cluster the vertices of $G_1$ and $G_2$ in the embedded space using GMM as employed by \texttt{MClust} \cite{mclust}.  Denote the cluster centers obtained from clustering $G_1$ (resp., $G_2$) via $\xi_1$ (resp., $\xi_2$).
   Modeling $G_1$ and $G_2$ via SBMs, estimate block probability matrices via $\hbb=\xi_1I_{p_1,q_1}\xi_1^T\in\mathbb{R}^{K_1\times K_1}$ and $\hbbc=\xi_2I_{p_2,q_2}\xi_2^T\in\mathbb{R}^{K_2\times K_2}$.
   For a proof of the Frobenius norm consistency of these estimates in the SBM setting, see Theorem \ref{thm:thm1} in Appendix \ref{app:b's}.
   \item[3.] 
    For $K_1\leq K_2$, then we will proceed to trim the graph in model space as follows. 
   \begin{itemize}
       \item[i.] First, align $\hbb$ to $\hbbc$ by finding
   $$P\in\text{argmin}_{Q\in\Pi_{K_1,K_2}}\|\hbb-Q\hbbc Q^T\|_
   F^2.$$
   where 
   $$\Pi_{K_1,K_2}=\{P\in\{0,1\}^{K_1\times K_2}\text{ s.t. }\vec 1_{K_2} P\leq \vec 1_{K_1}, P\vec 1_{K_2}= \vec 1_{K_1}\}.$$
   Note that while solving the above problem is NP-hard in general, there are computationally feasible options for relatively small $K_1$.
   \item[ii.] For $P$ in the above argmin (which need not be unique in general), denote $\xi_{2,t}:=P\xi_2$.
   Letting the collection of vertices whose blocks were selected by $P$ be denoted $\mathcal{I}$ (i.e., whose blocks were 
  \emph{not} trimmed), we have that the trimmed network is $G_2[\mathcal{I}]$, with estimated block probability matrix (abusing notation) denoted $\hbbct=\xi_{2,t}I_{p_2,q_2}\xi_{2,t}^T$.
    \end{itemize}
    \item[4.] Embed $G_2[\mathcal{I}]$ into $\mathbb{R}^{d_1}$ using ASE.
    Denoting the embeddings of $G_1$ and $G_2[\mathcal{I}]$ via $\hbX_1$ and $\hbX_2$ respectively.
    \item[5.] As we know which cluster centers in $G_1$ have been aligned via graph matching to those in $G_2[\mathcal{I}]$, we can use this information to align the graphs in the embedded space without the need for seeds.  To wit, solve the indefinite orthogonal Procrustes problem
    $$\mathbf{W}\in\text{argmin}_{\mathbf{O}\in\mathcal{O}_{p_1,q_1}}\|\xi_1\mathbf{O}-\xi_{2,t}\|_F,$$ and
    set $\hbX_{1,a}=\hbX_1 \mathbf{W}$ (where $\mathcal{O}_{p_1,q_1} = \{\mathbf{M} \in \mathds{R}^{d_1 \times d_1} \text{s.t. } \mathbf{M}^T I_{p_1, q_1} \mathbf{M} = I_{p_1, q_1}\}$ is known as the indefinite orthogonal group).
    See \cite{indefinite} for an approach to approximately solving the indefinite Procrustes problem.
    In many cases, it is appropriate to solve instead the orthogonal Procrustes problem 
    in which we seek 
    $$\mathbf{W}\in\text{argmin}_{\mathbf{O}\in\mathcal{O}_{d}}\|\xi_1\mathbf{O}-\xi_{2,t}\|_F,$$ 
    rather than the indefinite version.
    Indeed, in our experiments below, we found it sufficient to use the orthogonal Procrustes solution here (which also obviates the need to estimate the extra parameters $p_1$ and $q_1$). 
    We present the algorithm here in its fullest generality for use in settings where the indefinite Procrustes problem is needed.
\item[6.] Cluster the rows of 
$Z=\begin{pmatrix}
\widehat X_{1,a}\\
\widehat X_2
\end{pmatrix}$
using GMM (here, we employ \texttt{MClust} again), and finally, rank the candidate matches in $G_2[\mathcal{I}]$ according to the following Mahalanobis-distance-based scheme. \label{mahalanobis_item}
\begin{itemize}
\item[i.] Let $u \in V(G_1)$ and $v \in V(G_2[\mathcal{I}])$ be clustered points in $G_1$ and $G_2[\mathcal{I}]$, and let $\Sigma_u$ and $\Sigma_v$ be their respective covariance matrices obtained by the GMM-based clustering. 
\item[ii.] Compute (where for a matrix $\mathbf{M}$, $\mathbf{M}^\dagger$ represents the Moore-Penrose pseudoinverse of $\mathbf{M}$)
        \begin{align}
        \label{eq:mahal}
            \bigtriangleup(u,v) = \max(D_u(u,v),D_v(u,v))
        \end{align}
        where
         \[D_u(u,v) = \sqrt{(u-v) \Sigma_u^{\dagger} (u-v)^T} \]
        \[ D_v(u,v) = \sqrt{(u-v) \Sigma_v^{\dagger} (u-v)^T} \]
        \item[iii.] Rank the vertices in $G_2[\mathcal{I}]$ by increasing value of $\min_{u \in S_1} \bigtriangleup(u,v)$, where $S_1$ is the set of vertices of interest in $G_1$.
        \end{itemize}
\end{itemize}
\noindent Seeded vertices can be computationally (and financially) expensive to obtain.
In addition to being more amenable to theoretical analysis, one of the principle benefits of the current regularization scheme is that it obviates the need for seeded vertices.  
Step 5.\@ in the above algorithm replaces the seeded Procrustes alignment needed in \cite{agterberg2019vertex} with an unseeded alignment of cluster centers. See Figure \ref{fig:seed_stuff} for a comparison of our regularization procedures with seeds and without.

\subsubsection{Trimming in model versus graph space}
\label{sec:old_v_new}
Since the trimming  \cite{agterberg2019vertex} happens in the graph space, we will refer to it as the ``graph trimming" method. Similarly, we will refer to our newly presented method as ``model trimming". 
We next explore the comparative performance of these two methods using the same simulation setup as that described in \cite{agterberg2019vertex} (albeit, with different noise levels). 
In our simulations, we consider a graph $G_1$ with $500$ core vertices and a contaminated graph $G_2$ with $500+m$ vertices (here $m$ represents the level of contamination, where we consider $m=200$ and $m=400$). 
To wit, the graphs are sampled from the following SBM model:
\begin{itemize}
    \item[i.] $G_1\sim \mathrm{SBM}(500,2, \bB,\vec n=(250,250),\nu=1)$
    where 
    $$\bB=\begin{bmatrix} 
    0.7 &0.2\\
    0.2& 0.3
    \end{bmatrix}$$
    Here we use the true $d_1=2$ in the ASE embedding.
    \item[ii.] $G_2\sim \mathrm{SBM}(500+m,6, \bB^{(c)},\vec n=(250,m/4,m/4,250,m/4,m/4),\nu=1)$
    where $\bB^{(c)}$ is generated from $\bB$ as described in Section~\ref{sec:blocknoise} with $s+ = s_{-} = 0.2$, i.e., 
    $$\bB^{(c)}=\begin{bmatrix} 
0.70& 0.76 &0.56& 0.20& 0.36& 0.16\\
0.76& 0.76 &0.70& 0.36& 0.36& 0.20\\
0.56& 0.70 &0.56& 0.16& 0.20& 0.16\\
0.20& 0.36 &0.16& 0.30& 0.44& 0.24\\
0.36& 0.36 &0.20& 0.44& 0.44& 0.30\\
0.16& 0.20 &0.16& 0.24& 0.30& 0.24
    \end{bmatrix}$$
    Here we use the true $d_2=6$ in the ASE embedding.
        \item[iii.] Letting $\mathcal{C}=[1:250,(250+m/2+1):(500+m/2)]$
    $(G_1,G_2[\mathcal{C}])\sim \mathrm{SBM}(500,2, \bB,\vec n=(250,250),\rho)$, we obtain two graphs which have correlated core vertex sets.
    \end{itemize}
    
    \noindent We analyzed the performance of the VN task after trimming under both methods: graph trimming of \cite{agterberg2019vertex} and model trimming (using, for ease of comparison, simple orthogonal Procrustes in step 3.i of our procedure as opposed to the generalized Procrustes solver); the results are summarized in Figure \ref{fig:oldnew1}.
    Note that the graph trimming method requires seeded vertices to run to completion and thus, for this experiment we had used 10 randomly chosen seeds from $\mathcal{C}$ to align the embeddings when doing graph trimming. 
    In each panel of the figure, we plot on the $y$-axis the number of vertices in $G_1$ (when considered as the vertex of interest) with their corresponding vertex of interest ranked in the top $x$. 
    In other words, for $x=k$, $y=f(x)$ gives us the number of vertices in $G_1$ which have their corresponding vertex of interest (in $G_2$) ranked  ${1^{st},2^{nd},...,\text{or},k^{th}}.$ 
    In the middle (resp., bottom) row, we plot the performance of the trimming method proposed herein (resp., the trimming method of \cite{agterberg2019vertex}).
    In the top row, we plot the difference in performance across the two methods (model trimming - graph trimming).
    In the first column (resp., second and third), we show performance for $m=200$ and $\rho=0.7$ (resp., $m=400$, $\rho=0.7$ and $m=200$, $\rho=0.9$).
    Each figure represents 30 Monte Carlo simulations (paired within each column), with performance in each simulation plotted in gray and the average over all MC plotted in red (top) or black (bottom two rows).  In the bottom two rows, the blue line represents chance performance.
    
    In an ideal setting, where all the noise is trimmed perfectly and the nomination task performs perfectly, we would expect a horizontal line at $y=k^*$, where $k^*$ is the number of vertices in $G_1$. 
    From the figure, we observe that a higher correlation provides better VN performance for both methods, as we would intuitively expect. Furthermore, in graphs with greater number of noise vertices, the model trimming method performs significantly better than the graph trimming method.
    In the graphs with less noise, the model trimming method still performs better on average but the performance difference is less pronounced. We emphasize that the model trimming method uses no seed vertices while the graph trimming method uses ten seed vertices. 
    The effectiveness of the graph trimming method in \cite{agterberg2019vertex} was demonstrated in graphs with a relatively smaller amount of noise, and we suspect that our current method is indeed preferable in high-noise settings.
    In all cases, both methods are significantly better than chance here.
    
 \begin{algorithm}[t!]
\SetAlgoLined
\KwData{$G_1$ and $G_2$}
\begin{enumerate}
    \item $\hbX_1 = \text{ASE}(G_1,d_1)$ and 
          $ \hbX_2= \text{ASE}(G_2,d_2)$ (See Definition \ref{ase_def}) and estimate $p_1$, $q_1$, $p_2$, $q_2$, where $p_1+q_1=d_1$ and $p_2+q_2=d_2$.
    \item Separately cluster the rows of $\hbX_1$ and $\hbX_2$ via \texttt{MClust} (see \cite{mclust}); Denote the cluster centers obtained from clustering $\hbX_1$ (resp., $\hbX_2$) via $\xi_1$ (resp., $\xi_2$); 
    \item Set $\hbb=\xi_1 I_{p_1,q_1}\xi_1^T\in\mathbb{R}^{K_1\times K_1}$ and $\hbbc=\xi_2 I_{p_2,q_2} \xi_2^T\in\mathbb{R}^{K_2\times K_2}$;
    \item Find $P\in\text{argmin}_{Q\in\Pi_{K_1,K_2}}\|\hbb-Q\hbbc Q^T\|_F$;
    \item For $P$ in the above argmin, denote $\xi_{2,t} := P \xi_{2}$. Let $$\mathcal{I}=\{v\in V_2\text{ s.t. }v's\text{ corresponding cluster center is selected by }P\};$$   redefine 
    $ \hbX_2= \text{ASE}(G_2[\mathcal{I}],d_1)$;
    \item Solve $\mathbf{W}\in\text{argmin}_{O\in\mathcal{O}_{p_1,q_1}}\|\xi_1 O-P\xi_2\|_F$ and set $\hbX_{1,a}=\hbX_{1}\mathbf{W}$;
    \item Cluster the rows of $[\widehat{X}_{1,a}^T|\widehat{X}_{2}^T]^T$ via \texttt{MClust}, and rank the vertices according to the Mahalanobis distance computed in Eq. \ref{eq:mahal}.
\end{enumerate}
\caption{Block regularization pseudocode}
\label{alg:block_reg}
\end{algorithm}

\begin{figure}[t!]
    \centering
    \includegraphics[width=1\textwidth]{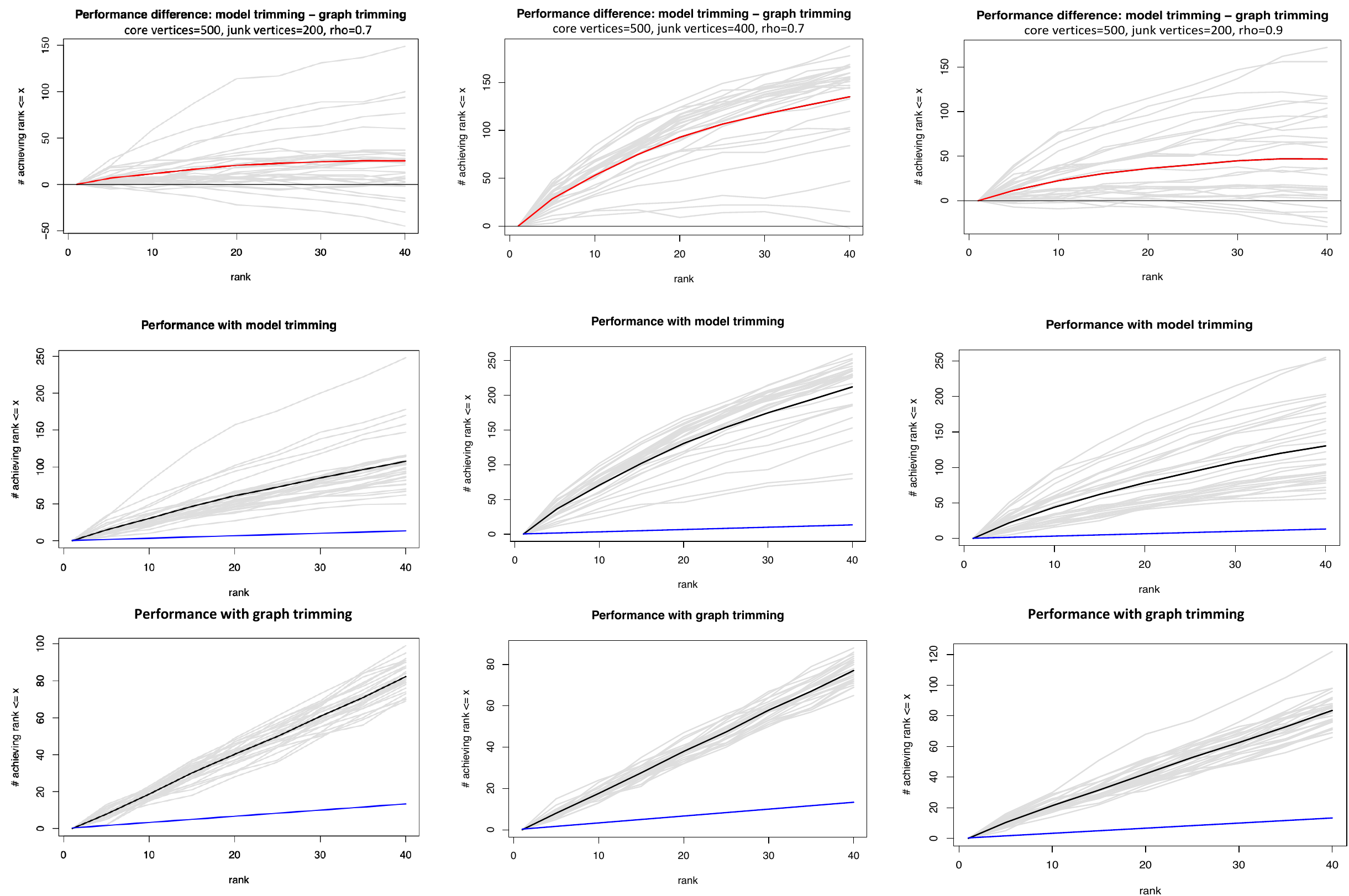}
    \caption{    In each panel of the figure, we plot on the $y$-axis the number of vertices in $G_1$ (when considered as the vertex of interest) with their corresponding vertex of interest ranked in the top $x$. 
    In the middle (resp., bottom) row, we plot the performance of the trimming method proposed herein (resp., the trimming method of \cite{agterberg2019vertex}).
    In the top row, we plot the difference in performance across the two methods.
    In the first column (resp., second and third), we show performance for $m=200$ and $\rho=0.7$ (resp., $m=400$, $\rho=0.7$ and $m=200$, $\rho=0.9$).
    Each figure represents 30 Monte Carlo simulations (paired within each column), with performance in each simulation plotted in gray and the average over all MC plotted in red (top) or black (bottom two rows).  In the bottom two rows, the blue line represents chance performance. Note that the range of the y-axis changes from figure to figure due to the difference in performance, number of noise vertices added, and vertices trimmed.
    } 
    \label{fig:oldnew1}
\end{figure}


\subsection{Diffuse noise contamination and regularization}
\label{sec:difcon}

\begin{figure}[t!]
    \centering
    \includegraphics[width=1\textwidth]{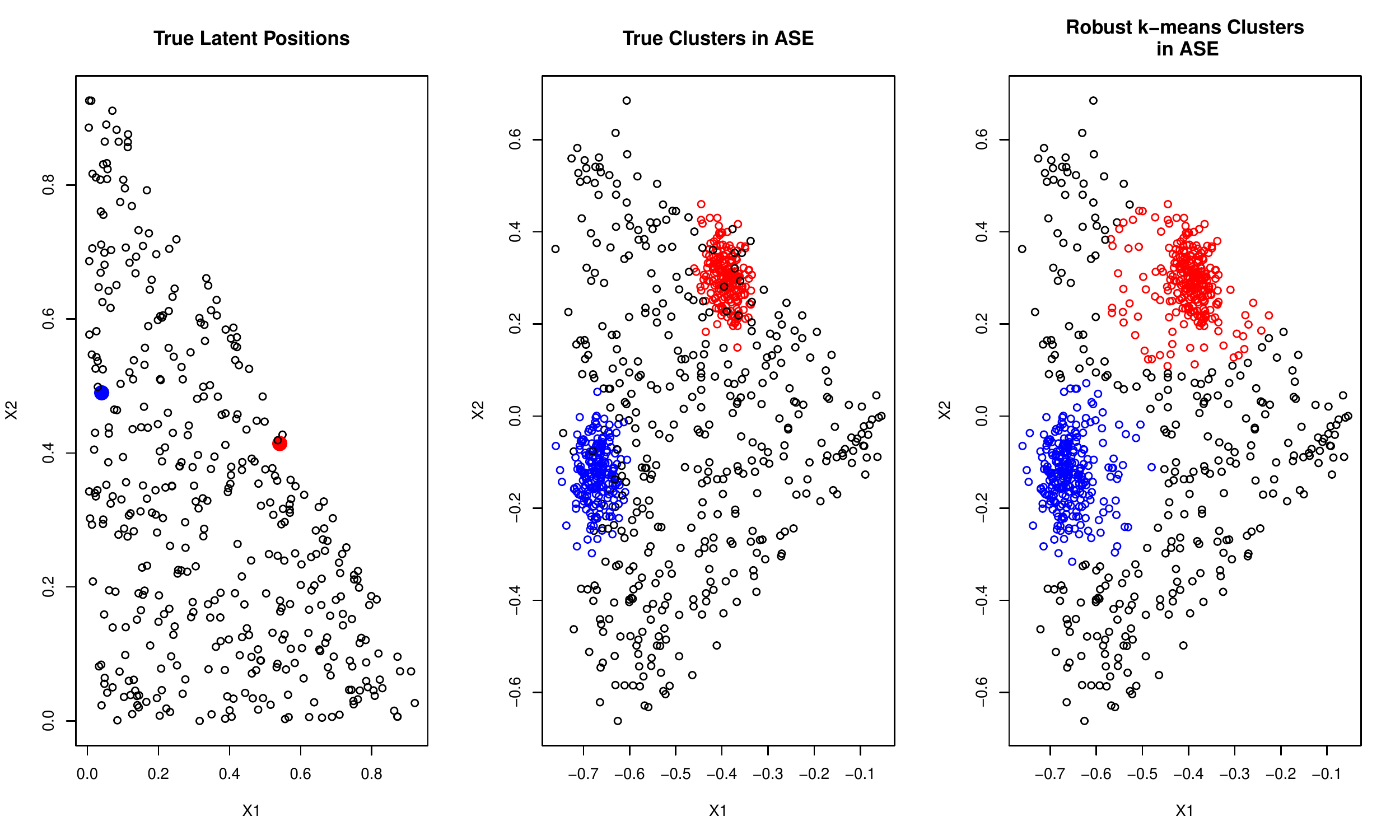}
    \caption{Diffuse noise contamination example in a 2-block SBM.  
    In the left panel, we plot the true latent positions of the signal (red/blue) and the noise (black); 
    in the center panel, we plot the estimated latent positions provided via ASE of the signal vertices (red/blue) and the noise (black)---note the rotation inherent to the GRDPG model; 
    in the right panel, we plot the clusters (in color) recovered by the robust $K$-means with $K=2$ and $\lambda=0.2$.
    }
    \label{fig:whitenoise}
\end{figure}

\noindent Although the above-described block regularization method is empirically (and theoretically) effective in the contaminated SBM setting above, in many real data settings the distribution of the noise is more nuanced or altogether unknown. 
Developing further regularization strategies to deal with broader noise models is a natural next step.
In this section we describe another technique, based on a robust K-means clustering method, that can be used in the setting where the contamination is unstructured (or less structured) diffuse noise.
We first describe the contaminated latent positions in the diffuse noise model, before giving a detailed description of our K-means based cleaning method.

Our starting point is that $G_2$ is a $k$-block SBM, that we further contaminate with unstructured noise.
To wit, let $\mathcal{X}_d$ be a subset of $\mathbb{R}^{d}$ such that $x^{\top} I_{p,q} y \in [0,1]$ for all $x,y \in \mathcal{X}_d$ and $\Omega$ be a convex subset of $\mathcal{X}_d$.  
    Next let $\{\mathfrak{z}^i\}_{i=1}^k$ be a collection of $k$ distinct points in $\mathcal{X}_d$. We then sample $n$ latent positions $\bY$ for a SBM graph from $F=\sum_{i=1}^k \pi_i\delta_{\mathfrak{z}^i}$ with $\pi_i>0$ for $i\in[k]$; note that $\bY\in\mathbb{R}^{n \times d}$ has at most $k$ distinct rows. 
   We then contaminate $\bY$ with i.i.d. ``white noise'' latent positions, $\mathbf{Z}\in\mathbb{R}^{m \times d}$, whose rows are i.i.d. uniformly distributed over $\Omega$. 
   Our contaminated model's latent positions can therefore be written as
    $\bX =[
\bY^T |
\bZ^T]^T\in\mathbb{R}^{(n+m) \times d}
$. 
The contaminated graph $G_2\sim$GRDPG($\bX,\nu$).
Note that the general nature of $\Omega$ makes $\bZ$ a natural model for more diffuse manifold noise contamination.

In order to regularize the diffuse noise out of $G_2$, we will employ a robust $K$-means clustering algorithm on $\hbX=$ASE($G_2,d$) as described below. 
Let $\mathcal{P}_2$ be the set of partitions of $[n+m]$ into two groups, and $\mathcal{C}_K$ the collection of sets of $K$ distinct points in $\mathcal{X}_d$.
We seek to solve the following optimization problem with tuning parameter $\lambda>0$,
\begin{equation}
\label{eq:opt}
\min_{\Phi\in \mathcal{C}_K, \mathfrak{p}\in \mathcal{P}_2} \underbrace{\Bigl(\sum_{i\in \mathfrak{p}_1}\min_{\phi\in \Phi}\|\phi-\hat X_i\|_2\Bigr)+\lambda |\{j:j\in\mathfrak{p}_2\}|}_{\Gamma(\Phi,\mathfrak{p})}.
\end{equation}
The partition of the vertices provided by $\mathfrak{p}\in\mathcal{P}_2$ divides the vertex set of $G_2$ into estimated signal vertices (i.e., those in $\mathfrak{p}_1$) and estimated noise vertices (those in $\mathfrak{p}_2$).
The vertices in $\mathfrak{p}_1$ contribute the the error in Eq. (\ref{eq:opt}) via the usual $K$-means term and are further clustered into $K$ disjoint groups.
Those vertices in $\mathfrak{p}_2$ are far from the cluster centers, and are not included in one of the $K$ clusters.
These unclustered points incur a constant cost (here $\lambda$) in Eq. (\ref{eq:opt}); $\lambda$ here is designed to penalize partitions that would cluster too few vertices in the noisy graph $G_2$ while also allowing for noise vertices to be excluded from the final clustering.
In practice, we can choose $\lambda$ based on the size of the clusters obtained by clustering the clean graph $G_1$; if $r$ is the largest cluster radius in $G_1$ 
then one simple heuristic is to choose $\lambda$ to be approximately $r+\log^2(n+m)/\sqrt{n+m}$ (see Eq. \ref{eq:IndConc2}) (we found $\lambda=0.2$ sufficient in most of our experiments).
The graph $G_2$ is then regularized by considering the ``cleaned" graph $G_2[\mathfrak{p}_1]$ in which the unclustered (ideally noise) vertices have been trimmed.
In the context of the above model, in Appendix \ref{sec:ktheory}, we establish the consistency of the optimal $K$-means clusters solving Eq.\@ (\ref{eq:opt}).
\begin{rem}
\emph{
In order to approximately solve Eq.\@ (\ref{eq:opt}), we adopt the following simple heuristic.  
We estimate a maximum cluster radius $r^*$ from a clustering of $G_1$, and use this optimal clustering in variant of the classical $K$-means procedure as follows:
\begin{itemize}
    \item[i.] Set $\mathfrak{p}_1=V, \mathfrak{p}_2=\emptyset$, and initialize the algorithm via an unconstrained $K$-means clustering of the rows of $\widehat X$;  Denote the cluster centers via $\{\mu_i\}_{i=1}^K$ and the cluster assignment vector via 
    $$b^{(K)}_v=\begin{cases}
    \text{argmin}_{i\in[K]}\|\widehat X_v-\mu_i\|&\text{ if }v\in\mathfrak{p}_1;\\
    0&\text{ if }v\in\mathfrak{p}_2;
    \end{cases}
    $$
    \item[ii.] Iterate the following two $K$-means adjacent steps until stopping conditions are met
    \begin{itemize}
        \item[a.] Update the $K$ cluster centers $$\mu_i=\frac{1}{|\{v\in \mathfrak{p}_1\text{ s.t. }b^{(K)}_v=i\}|}\sum_{\{v\in \mathfrak{p}_1\text{ s.t. }b^{(K)}_v=i\}}\widehat X_v$$
        \item[b.] For each $v\in V(G_2)$, if 
        $$\min_i \|X_v-\mu_i|<r^*,
        $$ set $v\in\mathfrak{p}_1$ and $b^{(K)}_v=i$.  
        Else, set $v\in\mathfrak{p}_2$ and $b^{(K)}_v=0$.
    \end{itemize}
\end{itemize}
Note that, in practice we iterate Step ii.\@ until the cluster assignments do not change or a maximum number of iterates has been met. 
We often run the above clustering multiple times (keeping the clustering minimizing Eq.~\eqref{eq:opt}) to account for randomness in the initialization of the $K$-means algorithm.
For an example of this algorithm in practice, see Figure \ref{fig:whitenoise}.}
\end{rem}

\subsection{Regularizing multiple noise sources}
\label{general_setting}
Heterogeneous and multimodal noise settings are very common when working with real data.
The two above-described cleaning methods can be strung together seamlessly to potentially ameliorate multiple noise sources simultaneously.
Consider the setting where $G_2$ is contaminated by both the block-noise of Section \ref{sec:blocknoise} and the diffuse noise of Section \ref{sec:difcon}.
This represents an idealized version of contamination that is both structured (here, designed to obfuscate the true graph model in model space) and unstructured.
Combining the two regularization strategies outlined above yields a two-step approach in which we first clean out the unstructured noise using the robust $K$-means method and then use our block-noise regularization algorithm to trim the structured noise. 
Below we present simulation experiments showing the performance of the VN task using our two-step cleaning procedure (again using regular orthogonal Procrustes rather than indefinite orthogonal Procrustes). 
Note that we will often compare performance post-regularization to a non-regularized VN procedure that operates via: using seeded vertices to align the embeddings of the two networks (without trimming) and nominating based on Step 6 of the procedure in Section  \ref{sec:blkreg}.

We first consider the contaminated SBM model from Section \ref{sec:old_v_new} with $m=1000$ noise vertices and $\rho=0.7$.
To this, we add $500$ additional white noise vertices sampled from (suitably rotated to yield feasible latent positions) the positive orthant in $\mathbb{R}^6$.
In our cleaning procedures, we consider $d_1=2$, $d_2=6$, and $K=6$.
We use the true number of clusters in the initial \texttt{MClust} clustering steps (Step 2 of Algorithm \ref{alg:block_reg}).
 Figure \ref{fig:2stage} shows the performance on the VN task after the implementation of the two-stage regularization procedure. 
 In the figure, we again plot on the $y$-axis the number of vertices in $G_1$ (when considered as the vertex of interest) with their corresponding vertex of interest ranked in the top $x$. 
 In the (L) panel, we plot the absolute performance of the 2-stage cleaning regularization procedure; in the (R) panel, we plot the difference in performance of the core spectral VN procedure post-regularization versus without regularization (post-pre).
 Note that the post-cleaning method does not require seeds, while the pre-cleaning method does (embed the graphs, seeded Procrustes to align, cluster and rank based on the computed Mahalanobis distance).
    Each grey line represents one of 30 Monte Carlo simulations, with the average performance over all MC plotted in black (L) or red (R).  The blue line represents chance performance.
We see that our 2-stage cleaning procedure is effective at mitigating the effect of the noise, even without seeds, and allows for significantly improved nomination performance versus running the core VN procedure sans cleaning.
This confirms our presumption that our regularization method proves to be effective in retrieving most of the original data signal, at least in simulation.

\begin{figure}[t!]
    \centering
    \includegraphics[width=1\textwidth]{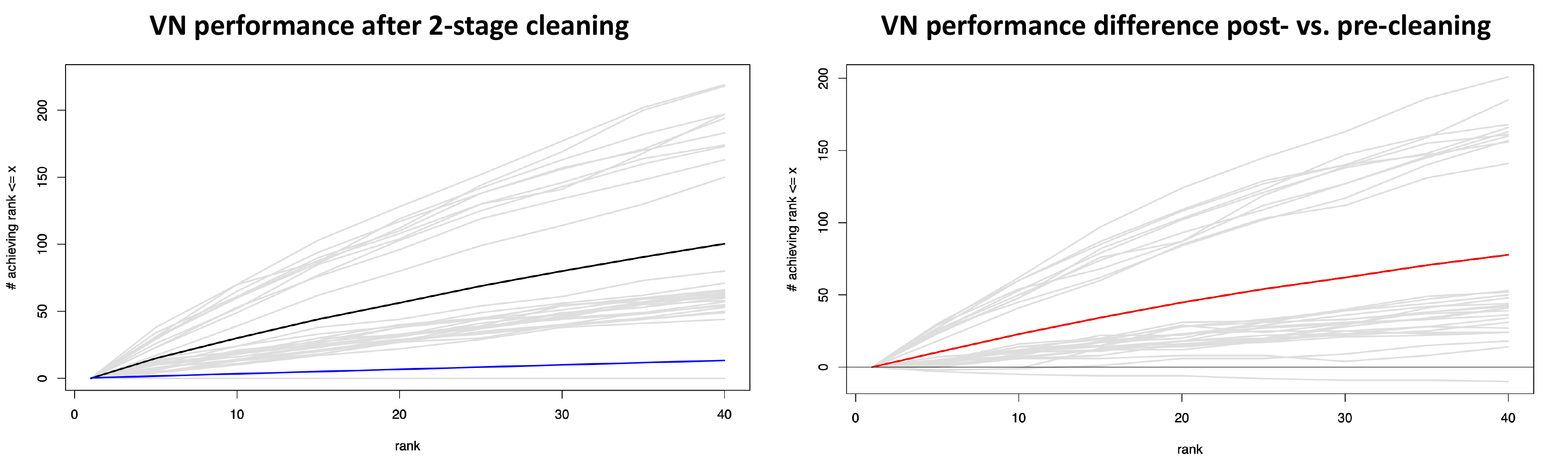}
    \caption{Vertex Nomination performance after the implementation of the two-stage regularization with $\lambda=0.2$. We plot on the $y$-axis the number of vertices in $G_1$ (when considered as the vertex of interest) with their corresponding vertex of interest ranked in the top $x$. 
    In the (L) panel, we plot the absolute performance of the 2-stage cleaning regularization procedure; in the (R) panel, we plot the difference in performance post-regularization versus without regularization (post-pre).
    Each grey line represents one of 30 Monte Carlo simulations, with the average performance over all MC plotted in black (L) or red (R).  The blue line represents chance performance.
    }
    \label{fig:2stage}
\end{figure}
Seeds are often an algorithmic \emph{luxury} and not always available in real-life settings. 
Therefore, presenting a version of our algorithm that does not require seeds is a highly advantageous task. 
A natural question though is what is lost (performance-wise) in the un-seeding?
While it is clear that optimally using seeds will always yield enhanced (or, at least, no worse) performance, this is not necessarily the case in the present algorithmic setting.
In the VN procedure outlined in Algorithm \ref{alg:block_reg}, seeds would be incorporated in Step 4, where seeded-Procrustes would replace the unseeded alignment of the cluster centers.
Surprisingly, incorporating seeds in this (natural) way yields significantly poorer algorithmic performance in simulations. With the above setup, we consider the effect of incorporating seeds in Figure \ref{fig:seed_stuff}.
 In each panel, we plot on the $y$-axis the number of vertices in $G_1$ (when considered as the vertex of interest) with their corresponding vertex of interest ranked in the top $x$. 
    In the (L) panels, we plot the absolute performance of the 2-stage cleaning regularization procedure with (top) and without (bottom) seeds; in the (R) panel, we plot the difference in performance with and without seeds (with-without).
    Each grey line represents one of 30 (paired) Monte Carlo simulations, with the average performance over all MC plotted in black (L) or red (R).  The blue lines on the left represent chance performance.
    We see significant performance improvement when applying the block-regularization procedure (as part of the 2-stage pipeline) without seeds; we postulate that this is because of the de-noising due to averaging in the estimated cluster centers versus the relatively noisy latent positions of the individual seeded vertices.
\begin{figure}[t!]
    \centering
    \includegraphics[width=1\textwidth]{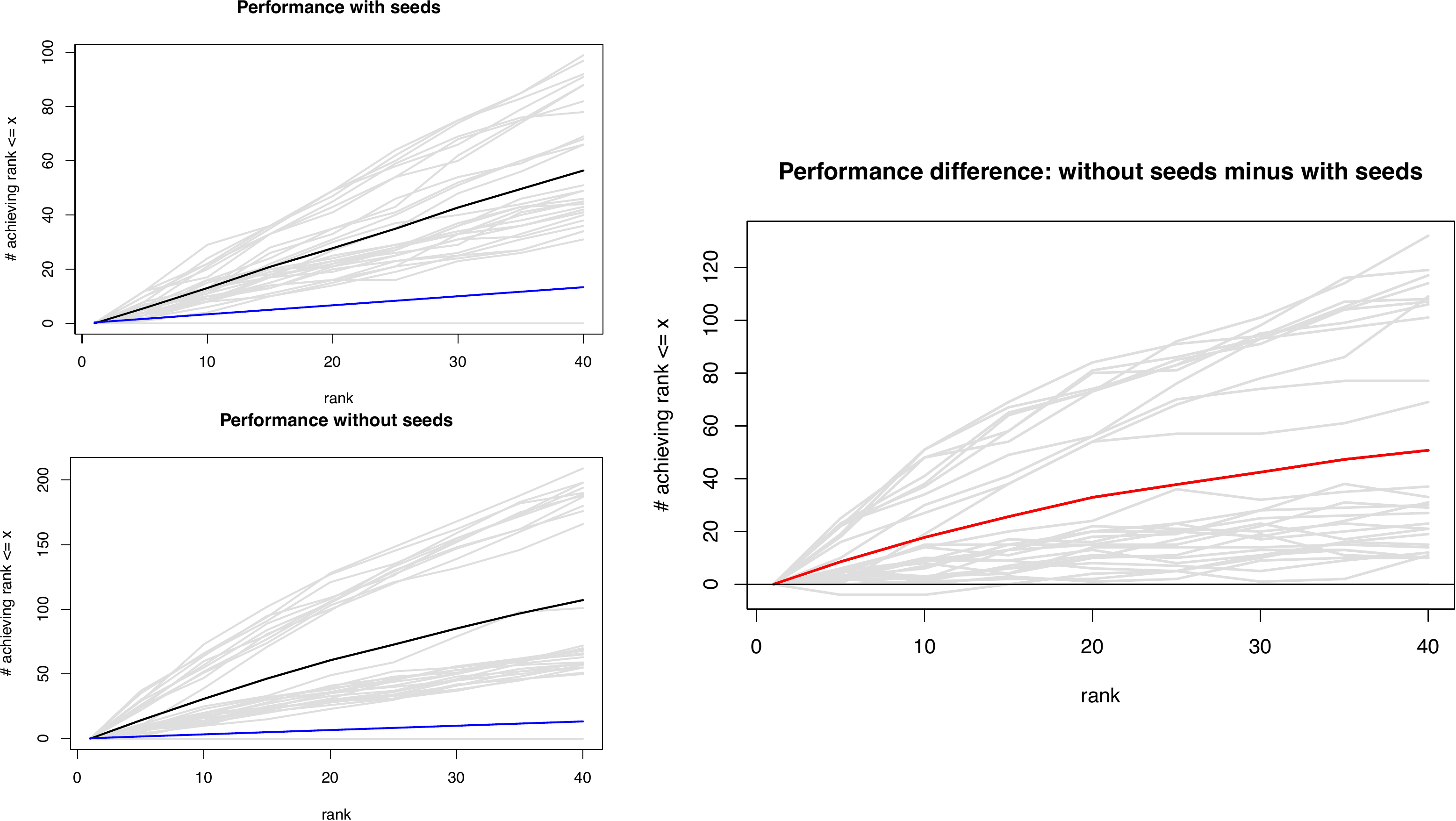}
     \caption{Vertex Nomination performance after with and without seeds. 
     In each panel, we plot on the $y$-axis the number of vertices in $G_1$ (when considered as the vertex of interest) with their corresponding vertex of interest ranked in the top $x$. 
    In the (L) panels, we plot the absolute performance of the 2-stage cleaning regularization procedure with (top) and without (bottom) seeds; in the (R) panel, we plot the difference in performance with and without seeds (with-without).
    Each grey line represents one of 30 (paired) Monte Carlo simulations, with the average performance over all MC plotted in black (L) or red (R), with $\lambda$ chosen to be $0.2$ in the robust $k$-means cleaning.  The blue lines on the left represent chance performance.}
    \label{fig:seed_stuff}
\end{figure}

\section{Real data experiments}
\label{sec:realdata}

In this section we present experiments based on three real data sets: the High School Friendship social network dataset from \cite{Friend}, connectomic Brain Data registered via the common DS data template \cite{brains}, and the political blogs data of \cite{adamic2005political}. In both these settings, we describe how the algorithms we propose are used to help mitigate the affect of an adversary in the respective settings. 

\subsection{VN on High School Friendship Social Networks}
\label{sec:friend}

\begin{figure}
    \centering
    \includegraphics[width=1\textwidth]{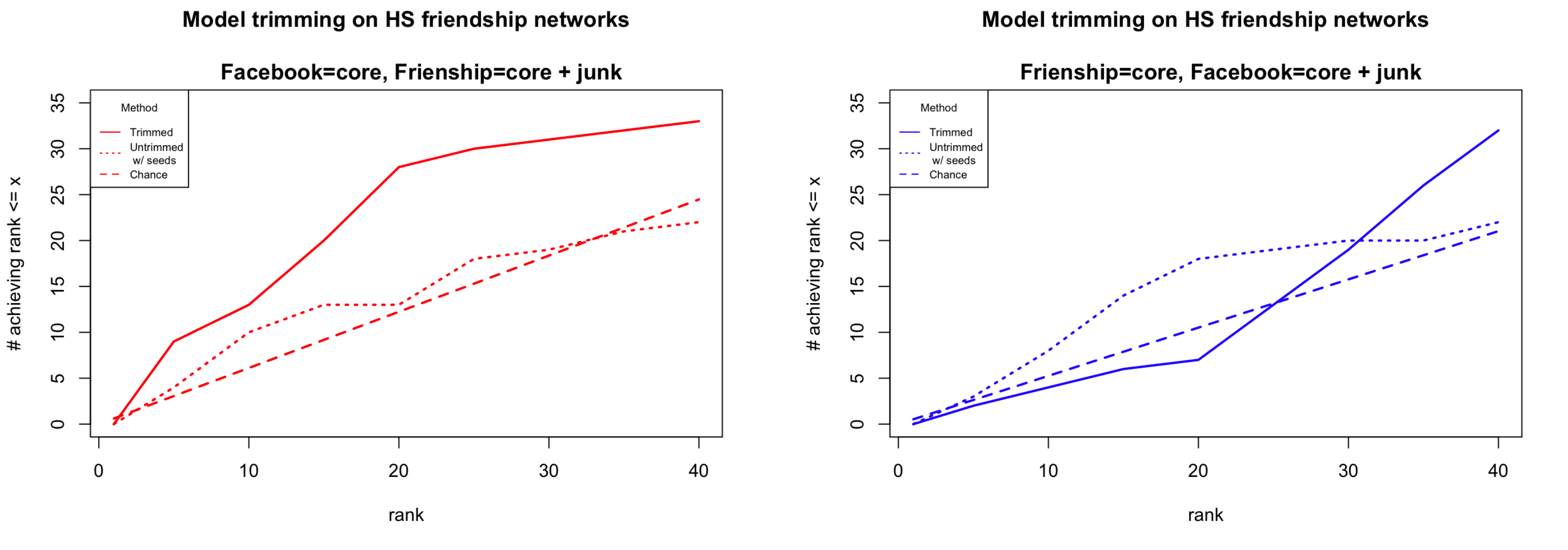}
    \caption{
  Block regularization for Vertex Nomination across the HS Friendship Social Networks.
  In each panel, we plot on the $y$-axis the number of vertices in the uncontaminated network (when considered as the vertex of interest) with their corresponding vertex of interest ranked in the top $x$. 
    In the left (resp., right) panel, the uncontaminated network is the core Facebook (resp., core Friendship) network and the contaminated graph is the full Friendship (resp., full Facebook) network.
    The solid lines represent performance post-cleaning, the dashed lines chance and the dotted lines performance without cleaning (with seeds).}
    \label{fig:two_friends}
\end{figure}

The data was collected from high school students at Lycée Thiers in Marseilles, France \cite{Friend}, and is composed of two different friendship social networks, each encapsulating a different interaction dynamic. 
\begin{itemize}
   \item[i.] \textbf{Facebook Friendship}\\
    Students were asked to use the Netvizz application which then created the network of Facebook friendship relations between the Facebook friends of each student who uses the application. 
    Since only 17 students gave access to their local network, it was not possible to build the entire network of Facebook relationships between the students through this information. Instead, a list of pairs of students, (``known-pairs" for which the existence or not of a friendship relation on Facebook was known), was used. The number of students considered here was 156.
   \item[ii.] \textbf{Self-reported Friendship}\\
Students were also asked to complete a survey in which they were asked to name their friends in the high school. We consider 134 friendship surveys in our analysis. 
\end{itemize}

\noindent Specifically, our data set consists of two graphs, $G_1$ (the self-reported friendship graph; we will refer to this as the Friendship network) and $G_2$ (the Facebook friendship graph), containing 134 and 156 vertices respectively.
In both graphs, the vertices represent students and the edges represent their friendship, and 
there is a core set of 82 vertices appearing in both graphs.
In $G_1$, two vertices are adjacent if at least one of the students reported to be friends with the other one, whereas in $G_2$ two vertices are adjacent if the corresponding students are friends on Facebook.

In this setting, we consider two experimental setups. In the first one, we only consider the core vertices in $G_1$ and treat the non-core vertices of $G_2$ as a contamination. We then use our new block-contamination trimming method to regularize $G_2$ and analyze the performance of our VN procedure on the regularized graph. 
In the second setting, we ``switch roles". We let $G_1$ be the contaminated graph and consider $G_2$ to be the clean graph, i.e. we only consider the core vertices of $G_2$. Results are summarized in Figure \ref{fig:two_friends}, and we note here that the method is implemented without seeds in both cases (contrasting with the similar VN analysis in \cite{patsolic2017vertex} that was seed dependent).
seems to depend on which graph is contaminated. 
From the figure , we see that performance here is highly dependent on the nature of the model ``noise." 
Indeed, when the contaminated graph is the Friendship graph, it is clear that the trimmed setting performs much better than both the 
untrimmed setting and chance. This seems to imply that our algorithm helps mitigate the affect of the adversary and retrieves most of the original graph structure in this setting.
In contrast, when the contaminated graph is the Facebook graph, the VN task performance after trimming is not that much better than before trimming (or even chance). Two plausible explanations for these differences are as follows. Firstly the optimum number of clusters chosen by \texttt{MClust} for the Facebook core (resp., junk) is 4 (resp., 8) and for the Friendship core (resp., junk) is 8 (resp., 7). Model trimming is thus possibly ineffective when $G_1$ is taken to be the Friendship core network while the contaminated $G_2$ is taken to be the Facebook graph as they both have $8$ estimated blocks.  The second explanation is that there is significant difference in the nature of online friendships versus reported friendships such as sampling bias inherent to friendship survey data \cite{Friend}.

\begin{figure}
    \centering
    \includegraphics[width=1\textwidth]{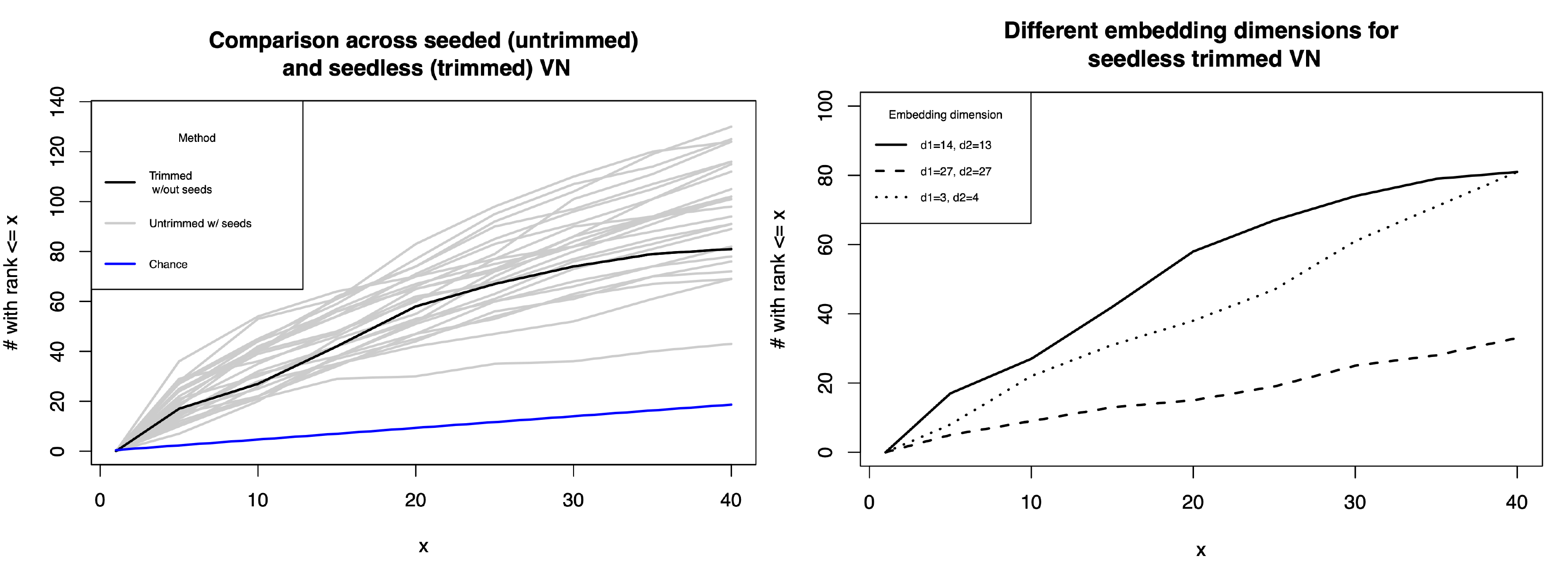}
    \caption{The left Figure shows the performance of the Vertex Nomination task in the brain data setting when trimming the noise vertices without seeds (grey lines) and when no trimming is performed and seeds are used. These performances are compared to chance. One can conclude that the trimmed setting seems to perform at least as good as the untrimmed setting, but without the need of seeds. The right Figure shows performance of the Vertex Nomination task for different embedding dimensions when computing the ASE of the graphs. } 
    \label{seednoseed}
\end{figure}

\subsection{Brain Data}
The data used in this experiment is a subset of the Connectivity-based Brain Imaging Research Database (C-BIRD) at Beijing Normal University (BNU). It contains data from 57 healthy young volunteers, 30 males and 27 males of ages 19-30, who completed two MRI scan sessions within an interval of approximately 6-weeks. A graph was generated from each of the two sessions. 
In our experiments, vertices of the brain graph represent voxel regions in the brain (after they have been registered to a common template), with edges measuring neural connection amongst regions. 
In our experiment setup, we consider two of the above-described brain scans of the same individual, which are inherently correlated, and define our graphs $G_1$ and $G_2$ as follows.
$G_1$ is a subgraph of the first brain scan graph, in which we only consider the vertices in the regions that lie in the left hemisphere of the individual. For $G_2$ we consider the whole second brain scan graph. Based on the regularization theory presented above, one can consider the second graph, $G_2$ to be a contaminated version of the graph $G_2'$, where $G_2'$ consists of the vertices in the second brain scan that occur in the regions that lie in the left hemisphere of the brain. The rest of the vertices can be considered ``contaminated'' vertices, having no analogue in $G_1$.
Note that one can consider a homology between the two hemispheres of the brain and thus the contamination described above can be thought of as structured noise rather than simply white noise.  
Our goal is to find the vertices in the second scan of the individual, which correspond to the vertices in the regions in the left hemisphere of the brain from the first scan.

There are 70 regions in each brain (derived via the Desikan atlas).  
Regions $1-35$ occur in the left hemisphere of the brain. Hence, we can easily create the induced subgraph $G_1$ by accessing the vertices that lie in these regions.  
Using \texttt{MClust} to estimate the block structure in $G_1$ and $G_2$, we use Algorithm \ref{alg:block_reg} to trim the ``block-structured'' noise from $G_2$.
We then compare the performance of our algorithm using \textit{no seeds} with the case when no trimming is performed and 5 seeds are being used to align the embedded networks 
Figure \ref{seednoseed} (left) shows the performance of the above described experiments with (gray) and without (black) seeds, compared to chance performance (blue). The experiments with seeds are repeated $25$ times, each time using a randomly selected seeds set. 
As one can easily deduce from the figure, both methods (trimmed with no seeds and untrimmed with seeds) perform better than chance.
Moreover, the trimmed setting often outperforms the untrimmed setting with seeds. This is encouraging, as seeds are expensive and hence being able to retrieve information without being provided seeds is a very much desired outcome.

A common source of error in experiments including embeddings arises when choosing the dimension of the embedding. We compare algorithmic performance of our experiment by choosing different embedding dimensions when computing the ASE of the graphs. In Figure \ref{seednoseed} (RIGHT) one can see that $d_1=14$ and $d_2=13$ (chosen by finding the second elbow of the SCREE plot of $G_1$ and $G_2$ \cite{athreya2017statistical}), outperforms embedding methods that underestimate (using the first SCREE elbow) or overestimate (using the third SCREE elbow) the embedding dimension.

\subsection{Political blogs}
\label{sec:blogs}

In this section, we consider the robust K-means procedure followed by nomination in the network of hyperlinks between weblogs on US politics \cite{polblog_paper}.
Specifically, the posts of 40 blogs were analyzed over the period of two months preceding the U.S. Presidential Election of 2004.
The blogs are natural segmented into two parts---liberal-leaning and conservative-leaning---and edges in the network capture how these blogs referred to each other both withing and across communities. 
\begin{figure}[t!]
    \centering
    \includegraphics[width=0.9\textwidth]{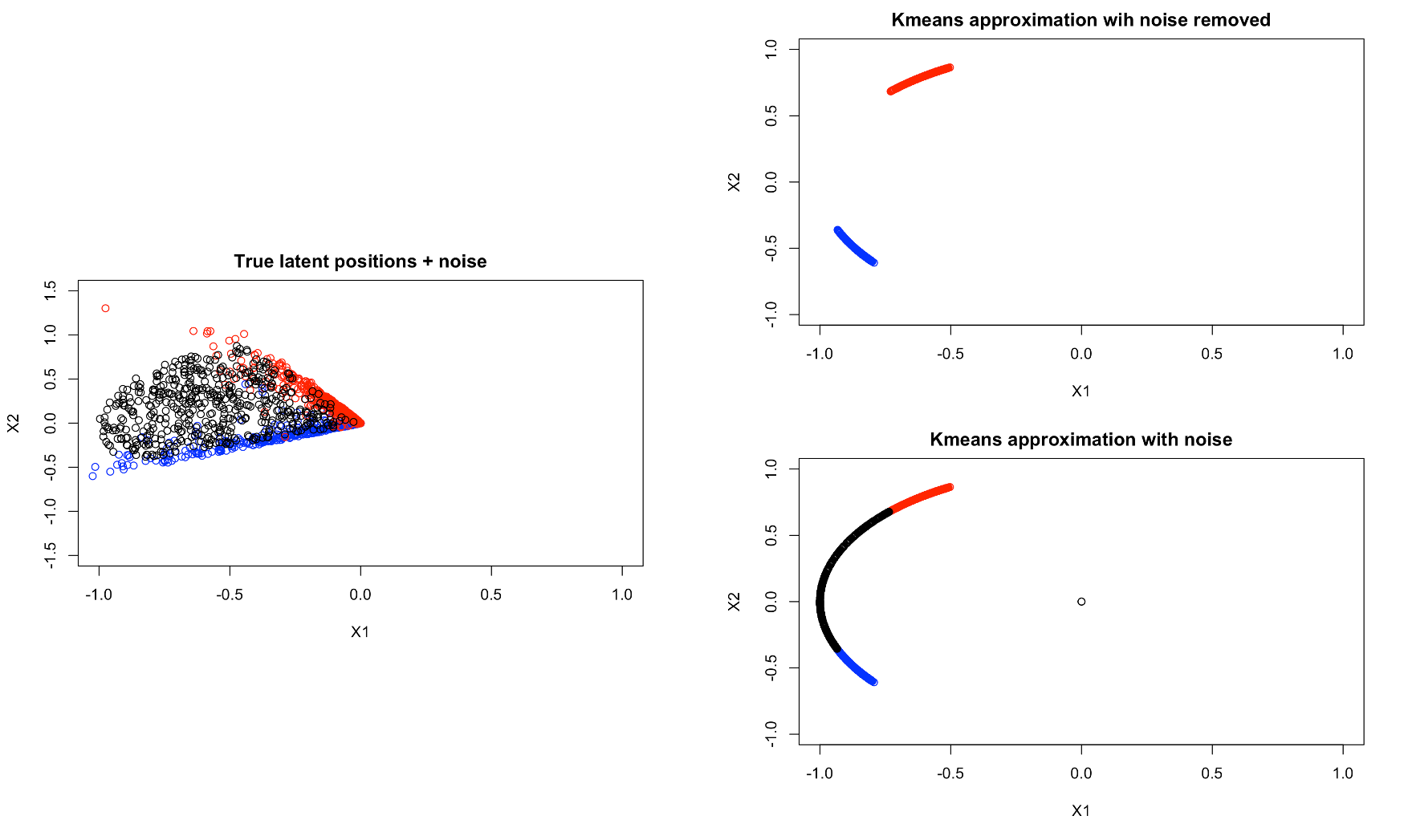}
    \caption{RIGHT: The true latent positions (liberal/conservative) plus added noise. LEFT: The clustering results (projected on the unit sphere) after the implementation of our kmeans method, with noise (above) and without the noise (below).}
    \label{fig:polblogs}
\end{figure}
In order to test the performance of our robust K-means algorithm, we adopt the following synthetic data approach.
We first consider the blog graph $G_1$ (the original data) as a 2-dimensional RDPG and estimate the latent positions of the network using ASE.
To these points, we add an additional sample of $m=500$ points 
drawn uniformly at random from the 2-sphere (first orthant) to artificially create noise data points. 
We then sample $G_2$ as an RDPG from the signal plus noise latent positions.
Here, $G_2$ represents a version of $G_1$ corrupted by two noise sources, the $m$ diffuse noise vertices and the noise from the RDPG resampling procedure.
We then proceed as follows:
\begin{itemize}
    \item[i.] Re-embed $G_2$ (call this $\hbX_2$), and project the embedded data to the sphere (this is a common tactic in clustering sparse graphs
    \cite{lyz,rubin2017statistical}); use the robust 2-means clustering to clean the noise vertices from $\hbX_2$.  Call the cleaned (unprojected) data $\hbX_{2,t}$.  The input and output of the 2-means clustering is plotted in figure \ref{fig:polblogs}.
    \item[ii.] Cluster the embedding of $G_1$ (call it $\hbX_1$) into 2 clusters and use orthogonal Procrustes to align these cluster centers to those obtained in the robust 2-means clustering of $\hbX_2$
    \item[iii.]  Use \texttt{MClust} to cluster the combined data (the aligned $\hbX_1$ and $\hbX_{2,t}$) and rank matches based on our Mahalanobis distance VN procure.
\end{itemize}
In Figure \ref{fig:polblogs_rank} we plot (in black) the precision at $k$ of our ranking scheme where we consider a positive match for precision purposes as follows: when nominating matches for a liberal (resp., conservative) blog, we consider a nominated blog of the same political persuasion as a positive match.
Chance precision is plotted in red.
In the figure, the precision at $k$ is averaged over all liberal (resp., conservative) blogs in the left (resp., right) panel, and we consider $k \in \{1,...,100\}$.
As can be seen on Figure \ref{fig:polblogs_rank}, we observe that our algorithm is performing well, and highly exceeds that of chance when considering both liberal and conservative vertices.

\begin{figure}[t!]
    \centering
    \includegraphics[width=0.9\textwidth]{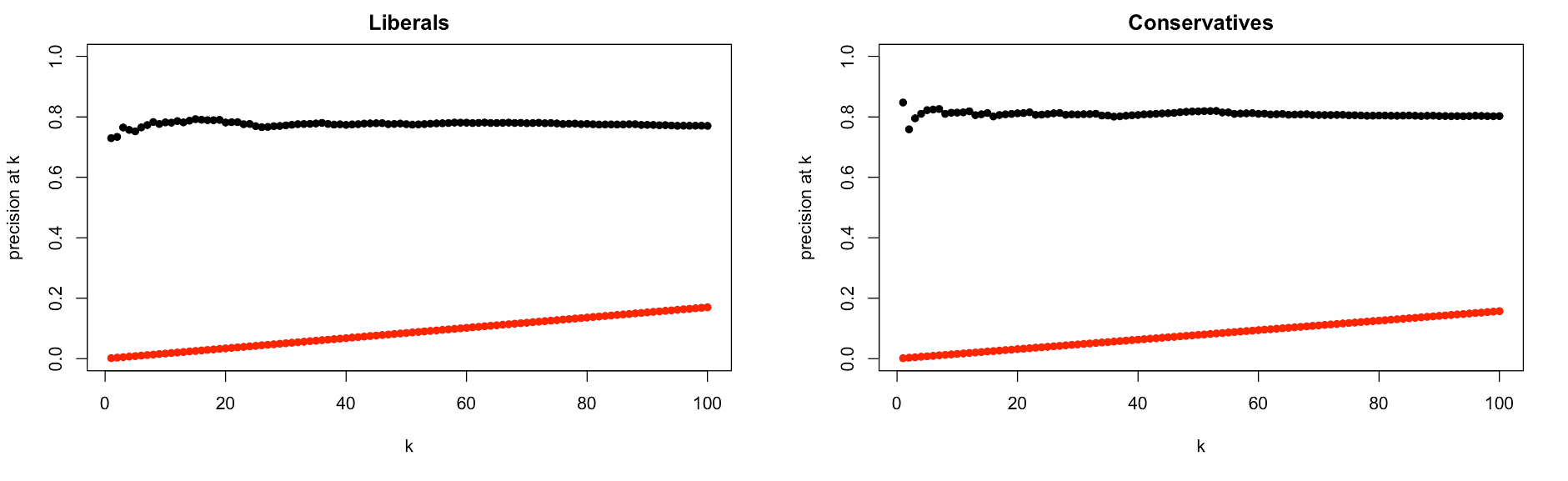}
    \caption{The mean precision at k of each vertex representing a liberal blog post (when considered as the vertex of interest) for different values of k, namely $k={1,...,50}$. Specifically, for each vertex of interest (when the vertex represents a liberal blog post), we calculate the number of other liberal blog posts that rank in the top k and divide this number by k. We proceed similarly for each vertex representing a conservative blog post}
    \label{fig:polblogs_rank}
\end{figure}

\section{Discussion and Future Work}
Adversarial attacks on networks are phenomena that unfortunately happen frequently and thus, trimming methods that help mitigate the affect of the contamination are much sought after. After presenting the adversarial model of \cite{agterberg2019vertex}, we introduced a novel trimming method, where the trimming happens in model space, rather than graph space. We extended this procedure to a setting where we might have two types of noise: structured and unstructured (white noise). 
Here, we further introduce a robust kmeans algorithm, where a fixed maximum cluster radius is used to gather points into clusters and``clean out" (ideally) the noise points, whose distance from any cluster center is greater than the maximum radius. A tuning parameter $\lambda$ helps with this process, by requiring the payment of a higher penalty in the objective function for the points that do not lie in any of the clusters. 
Experiments show that the two cleaning methods combine seamlessly and succeed in retrieving adequate information of the contaminated graph in both simulated and real-life data.

The above exploration is only the tip of the iceberg in an area where a lot of research can and should be done. 
The Stochastic Block Model is a simple, yet rich model to work with, but theory can be developed for more general structured noise settings. 
A natural next step would be exploring regularization methods in other, more complex contaminated models. 
In the latter case, an option might be to provide a few different robust algorithms, created precisely to be implemented in different contexts, and see which one gives the best results. 
Note that the regularization methods mentioned above act globally on the graph, but one can also consider local regularizers, which is another open area of research. 
For contamination procedures that act locally, regularization methods performed at a local level might prove to be effective in many settings. 
Another path to explore is lifting the above problem to higher dimensions, including multilayered graphs and time series. 
In this case, tensors are a natural framework to work in and provide the relevant theory that will support a thorough exploration of the subject. 
\vspace{3mm}

\noindent{\bf Acknowledgment:} This material is based on research sponsored by the Air Force Research Laboratory and Defense
Advanced Research Projects Agency (DARPA) under agreement number FA8750-20-2-1001. The U.S. Government is
authorized to reproduce and distribute reprints for Governmental purposes notwithstanding any copyright notation thereon.
The views and conclusions contained herein are those of the
authors and should not be interpreted as necessarily representing the official policies or endorsements, either expressed or
implied, of the Air Force Research Laboratory and DARPA or
the U.S. Government.
\bibliographystyle{plain}
\bibliography{mybib}

\begin{thebibliography}{10}

\bibitem{adamic2005political}
L.~A Adamic and N.~Glance.
\newblock The political blogosphere and the 2004 us election: divided they
  blog.
\newblock In {\em Proceedings of the 3rd international workshop on Link
  discovery}, pages 36--43, 2005.

\bibitem{polblog_paper}
L.~A. Adamic and Glance N.
\newblock The political blogosphere and the 2004 u.s. election: Divided they
  blog.
\newblock {\em Proceedings of the WWW-2005 Workshop on the Weblogging
  Ecosystem}, 2005.

\bibitem{agterberg2019vertex}
J.~Agterberg, Y.~Park, J.~Larson, C.~White, C.~E. Priebe, and V.~Lyzinski.
\newblock Vertex nomination, consistent estimation, and adversarial
  modification.
\newblock {\em Electronic Journal of Statistics}, 14(2):3230--3267, 2020.

\bibitem{agterberg2020nonparametric}
J.~Agterberg, M.~Tang, and C.~Priebe.
\newblock Nonparametric two-sample hypothesis testing for random graphs with
  negative and repeated eigenvalues.
\newblock {\em arXiv preprint arXiv:2012.09828}, 2020.

\bibitem{agterberg2020two}
J.~Agterberg, M.~Tang, and C.~E. Priebe.
\newblock On two distinct sources of nonidentifiability in latent position
  random graph models.
\newblock {\em arXiv preprint arXiv:2003.14250}, 2020.

\bibitem{airoldi2008mixed}
E.~M. Airoldi, D.~M. Blei, S.~E. Fienberg, and E.~P. Xing.
\newblock Mixed membership stochastic blockmodels.
\newblock {\em Journal of Machine Learning Research}, 2008.

\bibitem{alyakin2020correcting}
A.~A. Alyakin, J.~Agterberg, H.~S. Helm, and C.~E. Priebe.
\newblock Correcting a nonparametric two-sample graph hypothesis test for
  graphs with different numbers of vertices.
\newblock {\em arXiv preprint arXiv:2008.09434}, 2020.

\bibitem{amini2013pseudo}
A.~A. Amini, A.~Chen, P.~J. Bickel, and E.~Levina.
\newblock Pseudo-likelihood methods for community detection in large sparse
  networks.
\newblock {\em Annals of Statistics}, 41(4):2097--2122, 2013.

\bibitem{athreya2017statistical}
A.~Athreya, D.~E. Fishkind, M.~Tang, C.~E. Priebe, Y.~Park, J.~T. Vogelstein,
  K.~Levin, V.~Lyzinski, and Y.~Qin.
\newblock Statistical inference on random dot product graphs: a survey.
\newblock {\em Journal of Machine Learning Research}, 18(1):8393--8484, 2017.

\bibitem{bickel2016hypothesis}
P.~J. Bickel and P.~Sarkar.
\newblock Hypothesis testing for automated community detection in networks.
\newblock {\em Journal of the Royal Statistical Society, Series B}, pages
  253--273, 2016.

\bibitem{cai2015robust}
T.~T. Cai and X.~Li.
\newblock Robust and computationally feasible community detection in the
  presence of arbitrary outlier nodes.
\newblock {\em Annals of Statistics}, 43(3):1027--1059, 2015.

\bibitem{cape2019two}
J.~Cape, M.~Tang, and C.~E. Priebe.
\newblock The two-to-infinity norm and singular subspace geometry with
  applications to high-dimensional statistics.
\newblock {\em Annals of Statistics}, 47(5):2405--2439, 2019.

\bibitem{indefinite}
J.~Chung, B.~Varjavand, J.~Arroyo, A.~Alyakin, J.~Agterberg, M.~Tang, C.~E.
  Priebe, and J.~T. Vogelstein.
\newblock Valid two-sample graph testing via optimal transport procrustes and
  multiscale graph correlation with applications in connectomics.
\newblock {\em Stat}, 11, 2022.

\bibitem{Coppersmith2014}
G.~Coppersmith.
\newblock Vertex nomination.
\newblock {\em Wiley Interdisciplinary Reviews: Computational Statistics},
  6(2):144--153, 2014.

\bibitem{CopPri2012}
G.~A. Coppersmith and C.~E. Priebe.
\newblock Vertex nomination via content and context.
\newblock {\em arXiv preprint arXiv:1201.4118}, 2012.

\bibitem{adverDL2}
H.~Dai, H.~Li, T.~Tian, X.~Huang, L.~Wang, J.~Zhu, and L.~Song.
\newblock Adversarial attack on graph structured data.
\newblock In {\em Proceedings of the 35th International Conference on Machine
  Learning}, pages 1115--1124, 2018.

\bibitem{edge2018trimming}
D.~Edge, J.~Larson, M.~Mobius, and C.~White.
\newblock Trimming the hairball: Edge cutting strategies for making dense
  graphs usable.
\newblock In {\em IEEE International Conference on Big Data}, pages 3951--3958,
  2018.

\bibitem{eichler2017complete}
K~Eichler, F~Li, AL~Kumar, Y~Park, I~Andrade, C~Schneider-Mizell, T~Saumweber,
  A~Huser, D~Bonnery, B~Gerber, et~al.
\newblock The complete wiring diagram of a high-order learning and memory
  center, the insect mushroom body.
\newblock {\em Nature}, 548(175-182):23, 2017.

\bibitem{entezari2020all}
N.~Entezari, S.~A. Al-Sayouri, A.~Darvishzadeh, and E.~E. Papalexakis.
\newblock All you need is low (rank) defending against adversarial attacks on
  graphs.
\newblock In {\em Proceedings of the 13th International Conference on Web
  Search and Data Mining}, pages 169--177, 2020.

\bibitem{fishkind2019seeded}
D.~E. Fishkind, S.~Adali, H.~G. Patsolic, L.~Meng, D.~Singh, V.~Lyzinski, and
  C.~E. Priebe.
\newblock Seeded graph matching.
\newblock {\em Pattern Recognition}, 87:203--215, 2019.

\bibitem{fishkind2015}
D.~E. Fishkind, V.~Lyzinski, H.~Pao, L.~Chen, and C.~E. Priebe.
\newblock Vertex nomination schemes for membership prediction.
\newblock {\em Annals of Applied Statistics}, 9(3):1510--1532, 09 2015.

\bibitem{Bing}
H.~S. Helm, A.~Basu, A.~Athreya, Y.~Park, J.~T. Vogelstein, M.~Winding,
  M.~Zlatic, A.~Cardona, P.~Bourke, J.~Larson, C.~White, and C.~E. Priebe.
\newblock Learning to rank via combining representations.
\newblock {\em arXiv preprint arXiv:2005.10700v2}, 2020.

\bibitem{hoff2002latent}
P.~D. Hoff, A.~E. Raftery, and M.~S. Handcock.
\newblock Latent space approaches to social network analysis.
\newblock {\em Journal of the American Statistical Association},
  97(460):1090--1098, 2002.

\bibitem{holland1983stochastic}
P.~W. Holland, K.~B. Laskey, and S.~Leinhardt.
\newblock Stochastic blockmodels: First steps.
\newblock {\em Social networks}, 5(2):109--137, 1983.

\bibitem{huber2004robust}
P.~J. Huber.
\newblock {\em Robust statistics}.
\newblock John Wiley \& Sons, 2004.

\bibitem{advCD}
J.~Jia, B.~Wang, X.~Cao, and N.~Z. Gong.
\newblock Certified robustness of community detection against adversarial
  structural perturbation via randomized smoothing.
\newblock In {\em Proceedings of The Web Conference 2020}, pages 2718--2724,
  2020.

\bibitem{adversarial_survey}
W.~Jin, Y.~Li, H.~Xu, Y.~Wang, Sh. Ji, Ch. Aggarwal, and J.~Tang.
\newblock Adversarial attacks and defenses on graphs: A review, a tool and
  empirical studies.
\newblock {\em ACM SIGKDD Explorations Newsletter Volume 22 Issue 2 December,
  pp 19}, 2020.

\bibitem{karrer2011stochastic}
B.~Karrer and M.~E.~J. Newman.
\newblock Stochastic blockmodels and community structure in networks.
\newblock {\em Physical Review E}, 83(1):016107, 2011.

\bibitem{kolaczyk2009statistical}
E.~D. Kolaczyk.
\newblock {\em Statistical Analysis of Network Data: Methods and Models}.
\newblock Springer, 2009.

\bibitem{kolaczyk2014statistical}
E.~D. Kolaczyk and G.~Cs{\'a}rdi.
\newblock {\em Statistical analysis of network data with R}, volume~65.
\newblock Springer, 2014.

\bibitem{brains}
R.~M. Lawrence, E.~W. Bridgeford, P.~E. Myers, G.~C. Arvapalli, S.~C.
  Ramachandran, D.~A. Pisner, P.~F. Frank, A.~D. Lemmer, A.~Nikolaidis, and
  J.~T. Vogelstein.
\newblock Standardizing human brain parcellations.
\newblock {\em Scientific Data, Volume 8, Article number: 78}, 2021.

\bibitem{le2017concentration}
C.~M. Le, E.~Levina, and R.~Vershynin.
\newblock Concentration and regularization of random graphs.
\newblock {\em Random Structures \& Algorithms}, 51(3):538--561, 2017.

\bibitem{bays}
D.~S. Lee and C.~E. Priebe.
\newblock Bayesian vertex nomination.
\newblock {\em arXiv:1205.5082v1}, 2012.

\bibitem{lei2016goodness}
J.~Lei.
\newblock A goodness-of-fit test for stochastic block models.
\newblock {\em Annals of Statistics}, 44(1):401--424, 2016.

\bibitem{levin2020role}
K.~Levin, C.~E. Priebe, and V.~Lyzinski.
\newblock On the role of features in vertex nomination: Content and context
  together are better (sometimes).
\newblock {\em arXiv preprint arXiv:2005.02151}, 2020.

\bibitem{adversarial_large_graphs}
J.~Li, T.~Xie, L.~Chen, F.~Xie, X.~He, and Z.~Zheng.
\newblock Adversarial attack on large scale graph.
\newblock {\em arXiv:2009.03488v2}, 2021.

\bibitem{lyzinski2016consistency}
V.~Lyzinski, K.~Levin, D.~E. Fishkind, and C.~E. Priebe.
\newblock On the consistency of the likelihood maximization vertex nomination
  scheme: Bridging the gap between maximum likelihood estimation and graph
  matching.
\newblock {\em Journal of Machine Learning Research}, 17(1):6206--6239, 2016.

\bibitem{lyzinski2017consistent}
V.~Lyzinski, K.~Levin, and C.~E. Priebe.
\newblock On consistent vertex nomination schemes.
\newblock {\em Journal of Machine Learning Research}, 20(69):1--39, 2019.

\bibitem{lyz}
V.~Lyzinski, D.~Sussman, M.~Tang, A.~Athreya, and C.~Priebe.
\newblock Perfect clustering for stochastic blockmodel graphs via adjacency
  spectral embedding.
\newblock {\em Electronic Journal of Statistics}, 8(2):2905--2922, 2014.

\bibitem{lyzinski2016community}
V.~Lyzinski, M.~Tang, A.~Athreya, Y.~Park, and C.~E. Priebe.
\newblock Community detection and classification in hierarchical stochastic
  blockmodels.
\newblock {\em IEEE Transactions on Network Science and Engineering},
  4(1):13--26, 2016.

\bibitem{marchette2011vertex}
D.~Marchette, C.~E. Priebe, and G.~Coppersmith.
\newblock Vertex nomination via attributed random dot product graphs.
\newblock In {\em Proceedings of the 57th ISI World Statistics Congress},
  volume~6, 2011.

\bibitem{Friend}
R.~Mastrandrea, J.~Fournet, and A.~Barrat.
\newblock Contact patterns in a high school: a comparison between data
  collected using wearable sensors, contact diaries and friendship surveys.
\newblock {\em PLOS One}, 10 e0136497, 2015.

\bibitem{mele}
A.~Mele.
\newblock A structural model of segregation in social networks.
\newblock {\em doi:10.1920/wp.cem.2010.3210}, 2013.

\bibitem{newman2016equivalence}
M.~E.~J. Newman.
\newblock Equivalence between modularity optimization and maximum likelihood
  methods for community detection.
\newblock {\em Physical Review E}, 94(5):052315, 2016.

\bibitem{olhede2014network}
S.~C. Ohlede and P.~J. Wolfe.
\newblock Network histograms and universality of blockmodel approximation.
\newblock {\em Proceedings of the National Academy of Sciences},
  111(41):14722--14727, 2014.

\bibitem{patsolic2017vertex}
H.~G. Patsolic, Y.~Park, V.~Lyzinski, and C.~E. Priebe.
\newblock Vertex nomination via seeded graph matching.
\newblock {\em Statistical Analysis and Data Mining: The ASA Data Science
  Journal}, 13(3):229--244, 2020.

\bibitem{peixoto2014hierarchical}
T.~P. Peixoto.
\newblock Hierarchical block structures and high-resolution model selection in
  large networks.
\newblock {\em Physical Review X}, 4(1):011047, 2014.

\bibitem{resnick1997recommender}
P.~Resnick and H.~R. Varian.
\newblock Recommender systems.
\newblock {\em Communications of the ACM}, 40(3):56--58, 1997.

\bibitem{rohe2011spectral}
K.~Rohe, S.~Chatterjee, and B.~Yu.
\newblock Spectral clustering and the high-dimensional stochastic blockmodel.
\newblock {\em Annals of Statistics}, 39(4):1878--1915, 2011.

\bibitem{rubin2017statistical}
P.~Rubin-Delanchy, J.~Cape, M.~Tang, and C.~E. Priebe.
\newblock A statistical interpretation of spectral embedding: the generalised
  random dot product graph.
\newblock {\em Journal of the Royal Statistical Society, Series B}, 2022+.

\bibitem{covid}
S~Saha, A.~K. Halder, S.~S. Bandyopadhyay, P.~Chatterjee, M.~Nasipuri, and
  S~Basu.
\newblock Computational modeling of human-ncov protein-protein interaction
  network.
\newblock {\em arXiv preprint arXiv:2005.04108v1}, 2020.

\bibitem{procr}
P.~H. Schonemann.
\newblock A generalized solution of the orthogonal procrustes problem.
\newblock {\em Psychometrika}, 31(1):1–--10, 1966.

\bibitem{mclust}
L.~Scrucca, M.~Fop, T.~B. Murphy, and A.~E. Raftery.
\newblock mclust 5: Clustering, classification and density estimation using
  gaussian finite mixture models.
\newblock {\em The R Journal}, 8/1:289--317, 2016.

\bibitem{stigler1973asymptotic}
S.~M. Stigler.
\newblock The asymptotic distribution of the trimmed mean.
\newblock {\em Annals of Statistics}, 1(3):472--477, 1973.

\bibitem{suss}
D.~L. Sussman, M.~Tang, D.~E. Fishkind, and C.~E Priebe.
\newblock A consistent adjacency spectral embedding for stochastic blockmodel
  graphs.
\newblock {\em Journal of the American Statistical Association},
  107(499):1119--1128, 2012.

\bibitem{sussman12:_univer}
D.~L. Sussman, M.~Tang, and C.~E. Priebe.
\newblock Consistent latent position estimation and vertex classification for
  random dot product graphs.
\newblock {\em IEEE Transactions on Pattern Analysis and Machine Intelligence},
  36:48--57, 2014.

\bibitem{tang2017semiparametric}
M.~Tang, A.~Athreya, D.~L. Sussman, V.~Lyzinski, Y.~Park, and C.~E. Priebe.
\newblock A semiparametric two-sample hypothesis testing problem for random
  graphs.
\newblock {\em Journal of Computational and Graphical Statistics},
  26(2):344--354, 2017.

\bibitem{tang2017nonparametric}
M.~Tang, A.~Athreya, D.~L. Sussman, V.~Lyzinski, and C.~E. Priebe.
\newblock A nonparametric two-sample hypothesis testing problem for random
  graphs.
\newblock {\em Bernoulli}, 23(3):1599--1630, 2017.

\bibitem{aseff}
M.~Tang, J.~Cape, and C.~E. Priebe.
\newblock Asymptotically efficient estimators for stochastic blockmodels: The
  naive {MLE}, the rank-constrained {MLE}, and the spectral.
\newblock {\em Bernoulli}, 28:1049--1073, 2022.

\bibitem{tang2018limit}
M.~Tang and C.~E. Priebe.
\newblock Limit theorems for eigenvectors of the normalized {L}aplacian for
  random graphs.
\newblock {\em Annals of Statistics}, 46(5):2360--2415, 2018.

\bibitem{tang2013universally}
M.~Tang, D.~L. Sussman, and C.~E. Priebe.
\newblock Universally consistent vertex classification for latent positions
  graphs.
\newblock {\em Annals of Statistics}, 41(3):1406--1430, 2013.

\bibitem{facebook}
A.~L. Traud, P.~J Mucha1, and M.~A. Porter.
\newblock Social structure of facebook networks.
\newblock {\em arXiv preprint arXiv:1102.2166v1}, 2011.

\bibitem{yoder2018vertex}
J.~Yoder, Li~C., H.~Pao, E.~Bridgeford, K.~Levin, D.~Fishkind, C.~Priebe, and
  V.~Lyzinski.
\newblock Vertex nomination: The canonical sampling and the extended spectral
  nomination schemes.
\newblock {\em Computational Statistics \& Data Analysis}, 145:106916, 2020.

\bibitem{young2007random}
S.~J. Young and E.~R. Scheinerman.
\newblock Random dot product graph models for social networks.
\newblock In {\em International Workshop on Algorithms and Models for the
  Web-Graph}, pages 138--149. Springer, 2007.

\bibitem{zhao2012consistency}
Y.~Zhao, E.~Levina, and J.~Zhu.
\newblock Consistency of community detection in networks under degree-corrected
  stochastic block models.
\newblock {\em Annals of Statistics}, 40(4):2266--2292, 2012.

\bibitem{zheng_vn}
R.~Zheng, V.~Lyzinski, C.~E. Priebe, and M.~Tang.
\newblock Vertex nomination between graphs via spectral embedding and quadratic
  programming.
\newblock {\em Journal of Computational and Graphical Statistics}, pages 1--15,
  2022.

\bibitem{zhu2006automatic}
M.~Zhu and A.~Ghodsi.
\newblock Automatic dimensionality selection from the scree plot via the use of
  profile likelihood.
\newblock {\em Computational Statistics \& Data Analysis}, 51(2):918--930,
  2006.

\bibitem{advDL}
D.~Z{\"u}gner, A.~Akbarnejad, and S.~G{\"u}nnemann.
\newblock Adversarial attacks on neural networks for graph data.
\newblock In {\em Proceedings of the 24th ACM SIGKDD International Conference
  on Knowledge Discovery \& Data Mining}, pages 2847--2856, 2018.

\end{thebibliography}

\appendix

\section{Supporting results and pseudocode}

Herein we collect the supporting and proofs.

\subsection{Consistency of block probability estimates}
\label{app:b's}

In the theory below, we will make extensive use of the result, Theorem 3, from \cite{rubin2017statistical}, which states that if both second moment matrices ($\mathbb{E}(\bX I_{p,q} \bX^T)$ for $\bX\sim F$ or $\bX\sim F^{(c)}$) are full rank, then there exists a sequence of
indefinite 
orthogonal matrices $\mathbf{Q}_n\in\mathcal{O}(p,q)$ 
(so that $\mathbf{Q}_n$ satisfies $\mathbf{Q}_n^TI_{p,q}\mathbf{Q}_n=I_{p,q}$) 
and a universal constant $\alpha>0$ such that if the sparsity factor $\nu_n=\omega\left(\frac{\log^{4\alpha}n}{n}\right)$, then
\begin{equation}
\label{eq:IndConc2}
\max_{i=1,2,\ldots,n}\|\mathbf{Q}_n\hat X_i-\nu_n^{1/2}X_i\|=O_{\mathbb{P}}\left( \frac{\log^{\alpha}n }{n^{1/2}}\right)
\end{equation}
The constant $\alpha$ appearing in Eq. \ref{eq:IndConc2} will be used throughout the appendix and proofs contained therein.

Let the true latent position matrix for $\mathbf{B}$ be denoted $\boldsymbol{\mu}\in\mathbb{R}^{K_1\times d_1}$ and of $\mathbf{B}^{(c)}$ be denoted $\boldsymbol{\mu}^{(c)}\in\mathbb{R}^{K_2\times d_2}$. 
Let the distribution over latent positions for the uncontaminated (resp., contaminated) network be denoted 
$$F=\sum_{i=1}^{K_1}\pi_i \delta_{\mu_i}
\text{ (resp., }F^{(c)}=\sum_{i=1}^{K_2}\pi^{(c)}_i \delta_{\mu^{(c)}_i}).$$
Suppose further that  $\min_i\pi_i=\Theta(1)$ and similarly for $\min_i\pi_i^{(c)}$.
We further assume that the latent positions in the uncontaminated model (and in the contaminated model) satisfy (for $\alpha>0$ defined in Eq. \ref{eq:IndConc2})
$$|\mu_i-\mu_j|=\omega\left(\frac{\log^{\alpha}n }{n^{1/2}} \right).$$

\noindent Our proof will proceed with a MSE-based (Mean Square Error-based) clustering heuristic rather than the (more difficult to analyze) GMM-based clustering in our algorithm. 
Note that the GMM-based clustering is more appropriate in the current setting \cite{rubin2017statistical}, and often achieves superior performance in application \cite{tang2018limit}.
The MSE clustering
of the rows of $\mathbf{\hat X}$ into $K$ clusters provides
\begin{align*}
\hat{\mathbf{C}} &= \min_{\mathbf{C}\in \mathcal{C}_K}
\|\mathbf{C}-\mathbf{\hat X}\|_F, \text{ where }\\
\mathcal{C}_K &:=\{\mathbf{C} \in \mathbb{R}^{n\times d}: \mathbf{C} \text{ has K distinct rows}\},
\end{align*}
as the optimal cluster centroids for the $K$ clusters.

As mentioned in the main text, we claim that $\hbb$ and $\hbbc$ are suitable estimates of our original block membership matrices $\mathbf{B}$ and $\mathbf{B}^{(c)}$. The following Theorem supports this claim. 
\begin{theorem}
\label{thm:thm1}
Let $\boldsymbol{\hat \mu}$ be the matrix of cluster centroids provided by the optimal MSE-based clustering of $\mathbf{\hat X}$ in the uncontaminated network (with $\boldsymbol{\hat \mu}^{(c)}$ defined analogously for the contaminated network).
Given the block probability matrices $\mathbf{B}$ and $\mathbf{B}^{(c)}$ for the uncontaminated and contaminated networks, with assumptions on $\nu_n$, $\pi_i$, $\pi_i^{(c)}$, $\mu_i$ and $\mu^{(c)}_i$ as above we have that
$$ \| \underbrace{\boldsymbol{\hat \mu}I_{p_1,q_1}\boldsymbol{\hat \mu}^T}_{:=\hat{\mathbf{B}}} - \nu_n\mathbf{B} \|_F = O_{\mathbb{P}}\left( \frac{K_1\log^\alpha n }{n}\right) \quad 
   \| \underbrace{(\boldsymbol{\hat \mu}^{(c)})I_{p_2,q_2}(\boldsymbol{\hat \mu}^{(c)})^T}_{:=\hat{\mathbf{B}}_c}  - \nu_n\mathbf{B}^{(c)} \|_F =O_{\mathbb{P}}\left( \frac{K_2 \log^\alpha n }{n}\right)$$
\end{theorem}

\begin{proof}

The assumption on $\min_i\pi_i$ is sufficient to ensure that (via a simple Hoeffding's inequality) $\min_i n_i=\min_i |\{v:b_v=i\}|=\Theta_{\mathbb{P}}(n)$.
The assumption that 
$$|\mu_i-\mu_j|=\omega\left(\frac{\log^{\alpha}n }{n^{1/2}} \right)$$
then ensures that the optimal MSE clustering yields consistent estimates of  block memberships. 
To wit let the estimated cluster assignment vector for the the optimal MSE clustering routine be denoted $\hat b_v$,
    then   
$\min_{\tau\in S_K}
|\{v \in [n] :  \tau(b_v)\neq \hat b_v\}| =o_{\mathbb{P}}(1)$ (note that the proof is essentially identical to the proof of Theorem 2.6 in \cite{lyz} and is thus omitted).

Adopting the notation of \cite{aseff}, we let (with $d=\text{rank}(\mathbf{B})$ assumed known)
$$\mathbf{A}=\hat{\mathbf{U}} \hat{\mathbf{\Lambda}} \hat{\mathbf{U}}^T+\hat{\mathbf{U}}_\perp \hat{\mathbf{\Lambda}}_\perp \hat{\mathbf{U}}_\perp^T
$$
be the eigendecompostion of $\mathbf{A}$, where the columns of $\hat{\mathbf{U}}\in\mathbb{R}^{n\times d}$ are the eigenvectors associated with the $d$ largest eigenvalues in modulus of $\mathbf{A}$.
Let the associated eigendecompostion of $\mathbf{P}:=\nu_n\mathbf{X}I_{p_1,q_1}\mathbf{X}^T$ (where $\nu_n^{1/2}\mathbf{X}\in\mathbb{R}^{n\times d}$ matrix of latent positions of $\mathbf{A}$) be given by
$$\nu_n\mathbf{X}I_{p_1,q_1}\mathbf{X}^T={\mathbf{U}} {\mathbf{\Lambda}} {\mathbf{U}}^T$$
Define also,
\begin{align*}
    \Pi_{\mathbf{U}}&={\mathbf{U}}{\mathbf{U}}^T;\quad \Pi_{\mathbf{U}}^\perp=\mathbf{I}_n-\Pi_{\mathbf{U};};\\
    \mathbf{P}^\dagger&={\mathbf{U}} {\mathbf{\Lambda}}^{-1} {\mathbf{U}}^T; \quad \mathbf{E}={\mathbf{A}}- {\mathbf{P}}. 
\end{align*}
For each $k\in K$, let the vector of cluster membership for class $k$ in $\mathbf{Y}$ be denoted via $\mathbf{s}_k$, so that $$\mathbf{s}_k(i)=\mathds{1}\{\mathbf{X}_i\text{ is in cluster }k\}.$$
Similarly, for each $k\in K$, let the vector of the estimated cluster membership for class $k$ obtained by clustering $\hat{\mathbf{X}}$ be denoted via $\hat{\mathbf{s}_k}$; with this notation 
$\xi_{1,k}=\frac{1}{\hat n_k}\hat{\mathbf{s}}_k^T\hat{\mathbf{X}}$.

The proof next proceeds with the following decomposition, adapted here from Equation (A.5) in \cite{aseff},
\begin{align}
    \frac{n}{\nu^{1/2}_n}& (\hat{\mathbf{B}}_{kl} - \nu_n\mathbf{B}_{kl})\notag\\
    &=\frac{n }{n_k n_l \nu_{n}^{1/2}} (\mathbf{s}_k^T \mathbf{E} \mathbf{\Pi}_{\mathbf{U}} \mathbf{s}_l + \mathbf{s}_l^T \mathbf{\Pi}_{\mathbf{U}}^{\perp} \mathbf{E} \mathbf{\Pi}_{\mathbf{U}} \mathbf{s}_k)\label{eq:ee1} \\
    & + \frac{n }{n_k n_l \nu_{n}^{1/2}} (\mathbf{s}_k^T \mathbf{\Pi}_{\mathbf{U}}^{\perp} \mathbf{E}^2 \mathbf{P}^{\dagger} \mathbf{s}_l + \mathbf{s}_l^T \mathbf{\Pi}_{\mathbf{U}}^{\perp} \mathbf{E}^2 \mathbf{P}^{\dagger} \mathbf{s}_k) \label{eq:ee2}\\
    & + O_{\mathbb{P}}(n^{-1/2} \nu_n^{-1})\notag
\end{align}
We have from \cite{aseff} that conditional on $\mathbf{P}$, $\nu_n^{1/2}$ times Eq. \ref{eq:ee2} converges a.s. to a constant.
Therefore $\frac{\nu_n^{1/2}}{\log^{\alpha}n}$ times times Eq. \ref{eq:ee2} is $o_{\mathbb{P}}(1)$.
From Equation A.24 in \cite{aseff}, we have that conditional on $\mathbf{P}$, $\mathbf{s}_k^T \mathbf{E} \mathbf{\Pi}_{\mathbf{U}} \mathbf{s}_l + \mathbf{s}_l^T \mathbf{\Pi}_{\mathbf{U}}^{\perp} \mathbf{E} \mathbf{\Pi}_{\mathbf{U}} \mathbf{s}_k$ is the sum of $n(n+1)/2$ independent mean 0 random variables with bounded variance (all variances bounded by $d$ for example).
A simple application of Markov's inequality implies that for any $\epsilon>0$,
\begin{align*}
    \mathbb{P}\left(\frac{\nu_n^{1/2}}{\log^{\alpha}n} \frac{n }{n_k n_l \nu_{n}^{1/2}} |\mathbf{s}_k^T \mathbf{E} \mathbf{\Pi}_{\mathbf{U}} \mathbf{s}_l + \mathbf{s}_l^T \mathbf{\Pi}_{\mathbf{U}}^{\perp} \mathbf{E} \mathbf{\Pi}_{\mathbf{U}} \mathbf{s}_k|>\epsilon\right)&\leq \frac{O(n(n+1)/2)}{\frac{\log^{2\alpha}n}{\nu_n} \frac{n^2_k n^2_l \nu_{n}}{n^2 }\epsilon^2}\\
    &=O_{\mathbb{P}}\left(\frac{1}{\log^{2\alpha}n} \right)=o_{\mathbb{P}}(1)
\end{align*}
We then have that 
$$\frac{\nu_n^{1/2}}{\log^{\alpha}n} \left(\frac{n}{\nu^{1/2}_n} (\hat{\mathbf{B}}_{kl} - \nu_n\mathbf{B}_{kl})\right)=
\frac{n}{\log^{\alpha}n} (\hat{\mathbf{B}}_{kl} -\nu_n \mathbf{B}_{kl})=o_{\mathbb{P}}(1)+\underbrace{O_{\mathbb{P}}\left(\frac{n^{-1/2}}{\nu_n^{1/2} \log^{\alpha}n}\right)}_{o_{\mathbb{P}}(1)}.$$
This implies then that 
$$
|\hat{\mathbf{B}}_{kl} -\nu_n \mathbf{B}_{kl}|=o_{\mathbb{P}}\left(\frac{\log^{\alpha}n}{n} \right)
$$
A union bound over all the entries of $\mathbf{B}$ completes the proof.
The proof for $\hat{\mathbf{B}}^{(c)}$ is analogous.
\end{proof}

Now that we have proven the above claim, we want to use the estimates $\hbb$ and $\hbbc$ to trim the contaminated vertices in $G_2$. We can accomplish this by solving the following graph matching type problem.  
For each $a,b\in\mathbb{Z}>0$, 
define
$$\Pi_{a,b}=\{P\in\{0,1\}^{a\times b}:P\vec 1_b=\vec 1_a,\, \vec 1_a^TP\leq\vec 1_b \}.$$
Note that there exists a $P^*\in\Pi_{K,3K}$ such that
$$\|\mathbf{B}-P^* \mathbf{B}^{(c)} (P^*)^T\|_F=0.$$
We then seek to show that with high probability,
 \begin{equation}
    \label{eq:mini}
   P^*=\text{argmin}_{P\in\Pi_{K,3K}}\|\hbb-P \hbbc P^T\|_F^2.
\end{equation}
This would imply that $P^*$ recovers the correspondence between the original, non-contaminated blocks in $\hbb$ and the uncontaminated portion of $\hbbc$, with the remaining blocks (which ideally capture the contamination) trimmed by our procedure.
This is formalized in the following result:


\begin{theorem}
\label{thm:Bhat}
   Let $\mathbf{B}$ and $\mathbf{B}^{(c)}$ be such that
   $$
   \min_{P\in\Pi_{K,3K}\setminus\{P^*\}}\|\mathbf{B}-P\mathbf{B}^{(c)}P^T\|_F=\Omega(1)
   $$
For $K$ fixed, we then have
$$\mathbb{P}\left(P^*=\text{argmin}_{P\in\Pi_{K,3K}}\|\hbb-P \hbbc P^T\|_F\right)\rightarrow 1$$
\end{theorem}
\begin{proof}
Recall that in our model, the dimensions of $\mathbf{B}$ and $\mathbf{B}^{(c)}$ are respectively $K \times K$ and $3K \times 3K$.
Then, from Theorem \ref{thm:thm1}, we have that:
 \begin{align}
 \label{Bound 1}
  \| \hbb - \nu_n\mathbf{B} \|_F =O_{\mathbb{P}}\left( \frac{K\log^{\alpha}n}{n}\right)
 \end{align}
 \begin{align}
  \label{Bound 2}
  \| \hbbc - \nu_n\mathbf{B}^{(c)} \|_F = O_{\mathbb{P}}\left( \frac{K\log^{\alpha}n}{n }\right)
 \end{align}
With $P^*$ such that $\|\mathbf{B}-P^*\mathbf{B}^{(c)}(P^*)^T\|_F=0$, we then have the following:
 \begin{align*}
\|\hbb-P^* \hbbc (P^*)^T\|_F &= 
 \|\hbb \!- \nu_nP^* \mathbf{B}^{(c)} (P^*)^T \!+ \nu_nP^* \mathbf{B}^{(c)} (P^*)^T \!- P^* \hbbc (P^*)^T \!- \nu_n\mathbf{B} + \nu_n\mathbf{B} \| \\
& \leq \|\hbb - \nu_n\mathbf{B} \|_F + \|P^* \hbbc (P^*)^T - \nu_nP^* \mathbf{B}^{(c)} (P^*)^T \|_F \\
&\hspace{10mm}+ \|\nu_n\mathbf{B}  - \nu_nP^* \mathbf{B}^{(c)} (P^*)^T  \|  \\
& =O_{\mathbb{P}}\left( \frac{K\log^{\alpha}n}{n}\right) 
\end{align*}
Consider $P\in\Pi_{K,3K}$ satisfying $P \neq P^*$.  
Then the Assumptions A1 and A2 are sufficient to guarantee 
\begin{align*}
    \|\hbb-P \hbbc P^T\|_F  &=\Omega(\nu_n)=\omega\left( \frac{K\log^{4\alpha}n}{n}\right). 
\end{align*}
For fixed $K$, the probability that  $\|\hbb-P \hbbc P^T\|_F>\|\hbb-P^* \hbbc (P^*)^T\|_F$ converges to $1$.
A union over the (finitely many) $P\in\Pi_{K,3K}\setminus{P^*}$ then yields the desired result.
\end{proof}

\begin{rem}
The condition that $\min_{P\in\Pi_{K,3K}\setminus\{P^*\}}\|\mathbf{B}-P\mathbf{B}^{(c)}P^T\|_F=\Omega(1)$ is implied by either one of the two following two sets of conditions holding (where we are considering the parametrization for the edge addition/deletion probabilities as $\nu_ns_+$ and $\nu_n s_-$ respectively):
   \begin{align}
   \label{eq:A1}
      \min_{i,j\,:\,i\neq j}|\mathbf{B}_{i,i}-\mathbf{B}_{j,j}| &=\Omega\left(1\right)\tag{A1.1}\\
       \min_{i,j}|\mathbf{B}_{i,i}-\mathbf{B}_{j,j}-s_+(1-\mathbf{B}_{j,j})| &=\Omega\left(1\right)\tag{A1.2}\\
             \min_{i,j}|\mathbf{B}_{i,i}-\mathbf{B}_{j,j}(1-s_-)| &=\Omega\left(1\right)\tag{A1.3}
   \end{align}
   or
   \begin{align}
   \label{eq:A2}
    \min_{i,j,k,\ell\,:\,i\neq j, k\neq\ell,\{i,j\}\neq\{k,\ell\}}|\mathbf{B}_{ij}-\mathbf{B}_{k\ell}|
       &=\Omega\left(1\right)\tag{A2.1}\\
       \min_{i,j,k,\ell\,:\,i\neq j, k\neq\ell}|\mathbf{B}_{ij}-\mathbf{B}_{k\ell}-s_+(1-\mathbf{B}_{k\ell})|
       &=\Omega\left(1\right)\tag{A2.2}\\
       \min_{i,j,k,\ell\,:\,i\neq j,k\neq\ell}|\mathbf{B}_{ij}-\mathbf{B}_{k\ell}(1-s_-)|
       &=\Omega\left(1\right)\tag{A2.3}
   \end{align}
   The first set of conditions would ensure sufficient disagreement on the matched diagonal of $\|\mathbf{B}-P\mathbf{B}^{(c)}P^T\|_F$ to guarantee the growth rate on 
   $\min_{P\in\Pi_{K,3K}\setminus\{P^*\}}\|\mathbf{B}-P\mathbf{B}^{(c)}P^T\|_F$, while the second set of conditions guarantees sufficient error on the matched off-diagonal of $\|\mathbf{B}-P\mathbf{B}^{(c)}P^T\|_F$.
\end{rem}
\subsection{Consistency of robust K-means clustering}
\label{sec:ktheory}
We first prove a consistency result in the dense setting (i.e., in the setting where $\nu_n=\theta(1)$, and then will outline how to adapt the results to sparser settings. In the dense setting, we will assume that the penalty parameter $\lambda=\lambda_{n,m}$ is bounded away from 0.\\

\noindent \textbf{Assumption 1:} 
\emph{With $\bX$ defined as in Section \ref{sec:difcon}, let the $K$ distinct rows of $\bY$ represent the latent position vectors in $\mathbb{R}^d$ for the signal blockmodel component of $G_2$.
We assume that there exists a constant $\eta>0$ such that if $Y_i$ and $Y_j$ are two distinct rows of $\bY$, then $\|Y_i-Y_j\|>\eta.$
Also assume that $m$ (the number of rows of $\bZ$) satisfies $m=o(n)$.} 
\vspace{2mm}

\noindent Again we use Theorem 3 from \cite{rubin2017statistical}, which states that there exists a sequence of indefinite orthogonal matrices $\mathbf{Q}_n\in\mathcal{O}(p,q)$ (so that $\mathbf{Q}_n$ satisfies $\mathbf{Q}_n^TI_{p,q}\mathbf{Q}_n=I_{p,q}$) and a constants $c,C>0$ such that in this dense framework, the following holds with high probability for sufficiently large $n$

\begin{equation}
\label{eq:IndConc}
\max_{i=1,2,\ldots,n}\|\mathbf{Q}_n\hat X_i-\nu_n^{1/2}X_i\|:=\epsilon_{n,m}\leq C \frac{\log^{c}(n+m) }{(n+m)^{1/2}}
\end{equation}
Below, we will drop the subscript on $\epsilon_{n,m}$ (as well as on $\lambda_{n,m}$) to ease notation.
Conditioning on the event in Eq. (\ref{eq:IndConc}), consider the partition $\pi$ such that $\pi_1=\{1,2,\cdots,n\}$ and $\pi_2=\{n+1,n+2,\cdots,n+m\}$, and consider $\Phi$ composed of the $K$ distinct rows of $\bY$.
For this choice of $\pi$ and $\Phi$, we have that
\begin{align}
\label{eq:goodbnd}
\Gamma(\Phi,\pi)\leq C n  \frac{\log^{c}(n+m)}{(n+m)^{1/2}}+\lambda m.
\end{align}
Next, we let $(\hat\Phi,\hat\pi)$ be an element of the
argmin of Eq.\@ (\ref{eq:opt}), and consider balls $\{\mathcal{B}_i\}_{i=1}^K$ of radius 
$$r:=  \min\left(\frac{\lambda}{3}, \frac{\eta}{6}\right)$$ 
about the $K$ distinct rows of $\bY$.
By Assumption 1, these balls are disjoint.
Under our assumption that the penalty $\lambda$ is bounded away from $0$, we have that $r+\epsilon < \lambda$ for $n$ sufficiently large.
For $n$ sufficiently large, we then observe the following: 
\begin{itemize}
\item[i.]
If a ball $\mathcal{B}_i$ contains no centers in $\hat\Phi$, then
\begin{align*}
\Gamma(\hat\Phi,\hat\pi)\geq& \min_{y\in\{\mathfrak{z}^i\}_{i=1}^k} \left(\big|\{j: Y_j=y\text{ and }j\in \hat \pi_1\}\big|\cdot(r-\epsilon)+\big|\{j: Y_j=y\text{ and }j\in \hat \pi_2\}\big|\cdot\lambda\right)\\
\geq& \min_i n_i\cdot\min(r-\epsilon,\lambda)\\
=&\Omega_P\left(n\cdot (\lambda-\epsilon)\right)
\end{align*}
where the final equality holds with high probability as each $n_i$ has a Binomial distribution with parameters $n$ and $\pi_i>0$.
This yields the desired contradiction as this asymptotically dominates Eq.\@ (\ref{eq:goodbnd}) for $n$ sufficiently large.
\item[ii.]
Note that if a ball $\mathcal{B}_i$ contains two or more centers of $\hat\Phi$, then there is at least one ball with no cluster centers and the desired contradiction follows from i.\@ above.
\end{itemize}
Observations i.\@ and ii.\@ above yield that each $\mathcal{B}_i$ contains exactly one of the centers in $\hat\Phi$.
This yields then that all of the signal vertices in the ASE are properly clustered according to $(\hat\Phi,\hat\pi)$, as signal points will not be unclustered (as this induces a penalty of $\lambda$ in $\Gamma(\hat\Phi,\hat\pi)$ while assigning the signal point to the cluster center closest to its true latent position induces a penalty of at most $r+\epsilon<\lambda$) or misclustered (as this induces a penalty of at least $\eta-2r>2\eta/3$ in $\Gamma(\hat\Phi,\hat\pi)$ while assigning the signal point to the cluster center closest to its true latent position induces a penalty of at most $r+\epsilon<2\eta/3$).

Next, we ask how many noise points are assigned (incorrectly) a cluster label.
Noise points will only be assigned a label if they are within $\lambda$ of a cluster center (and hence, their true latent position is within $\lambda+r+2\epsilon$ of a true signal latent position). 
This probability is bounded above by (where $\Gamma(\cdot)$ is the usual $\Gamma$ function and $K$ is a universal constant) 
$$\mathfrak{p}=\frac{K\cdot(\lambda+r+2\epsilon)^d \pi^{d/2}/\Gamma(d/2+1) }{\frac{1}{2^d}\pi^{d/2}/\Gamma(d/2+1) }=K\cdot(2\lambda+2r+4\epsilon)^d.
$$
A simple application of Hoeffding's inequality yields that, with high probability, at most 
$$O(m \cdot K\cdot(8\lambda/3)^d)$$ 
(and hence $o(n)$) noise points fall within $\lambda+r+2\epsilon$ of a true signal latent position, and hence are assigned a cluster label.

\begin{rem}\emph{
In the sparse setting all the results above still hold, the only difference being that all parameters are scaled by $\sqrt \nu$. 
Indeed, the above argument follows mutatis mutandis (where the `$s$' subscript denotes the sparse parameter and a `$\mathfrak{d}$' subscript the dense) with:
\begin{align*}
    \lambda_s&=\sqrt \nu \lambda_{\mathfrak{d}}\\
    r_s &= \sqrt \nu r_{\mathfrak{d}}\\
    \eta_s &= \sqrt \nu \eta_{\mathfrak{d}}
\end{align*}
Note that $\epsilon$ is defined the same way in the sparse and dense settings, and that here $\lambda_s\gg\epsilon$ (assuming $\lambda_{\mathfrak{d}}$ is bounded away from 0) as
\begin{align*}
\epsilon_{n,m} = O_P\left( \frac{\log^{c}(n+m) }{(n+m)^{1/2}}\right)
\end{align*}
and 
\begin{align*}
    \nu_n = \omega\left(\frac{\log^{4c}n}{n}\right).
\end{align*}
}
\end{rem}
\end{document}